\documentclass{article}
\usepackage{bbm}
\usepackage{caption}
\usepackage{rotating}
\usepackage[maxfloats=100]{morefloats}
\usepackage{listings}
\usepackage{booktabs}
\usepackage{xcolor}
\usepackage[title]{appendix}
\usepackage{makecell} 
\usepackage{multirow}

\usepackage[symbol]{footmisc}

\usepackage{longtable}
\usepackage{subcaption}
\usepackage[numbers]{natbib}
\usepackage{geometry}

\usepackage{verbatim}
\usepackage{xcolor}
\usepackage{graphicx}
\usepackage{enumitem}
\usepackage{comment}
\usepackage{subfiles}

\usepackage{todonotes}
\usepackage{tikz-qtree,tikz-qtree-compat}
\usepackage{scrextend}
\usepackage{framed}
\usepackage{placeins}
\usepackage{setspace}
 \usepackage[ruled,vlined]{algorithm2e}
\usepackage{mathtools,collcell,eqparbox}

\usepackage{multicol}
\usepackage{pgfplots}
\usepackage{tikz}
\usepackage{listings}
\usepackage{courier}
\usepackage{url}
\usepackage{color}
\usepackage[title]{appendix}
\usepackage{authblk}

\usepackage{amssymb}
\usepackage{amsthm}
\usepackage{ dsfont }
\usepackage{amsmath}

\usepackage[pagebackref = true]{hyperref}
\hypersetup{
  colorlinks = true,
  linkcolor = teal,
  anchorcolor = teal,
  citecolor = teal,
  filecolor = teal,
  urlcolor = teal
  }

\pgfplotsset{compat=1.15}
\numberwithin{equation}{section}

\newcounter{BMatrix}

\newcommand{\setmaxwd}[1]{%
 \eqmakebox[BM-\theBMatrix][\BMalign]{$#1$}%
}
\MHInternalSyntaxOn

\MHInternalSyntaxOff
\makeatother

\newtheorem{theorem}{Theorem}
\newtheorem{assumption}{Assumption}

\newtheorem{proposition}{Proposition}[section]

\newtheorem{lemma}{Lemma}

\theoremstyle{definition}
\newtheorem{definition}{Definition}

\theoremstyle{remark}
\newtheorem{remark}{Remark}
\newtheorem*{remark*}{Remark}

\DeclareMathOperator*{\esssup}{ess sup}
\DeclareMathOperator*{\essinf}{ess inf}

\newcommand{\1}{\mathds{1}}

\newcommand{\A}{\mathbb{A}}
\newcommand{\W}{\mathbb{W}}
\newcommand{\R}{\mathbb{R}}

\newcommand{\X}{\mathbb{X}}

\newcommand{\Z}{\mathbb{Z}}

\newcommand{\E}{\mathbb{E}}

\newcommand{\veps}{\varepsilon}

\newcommand{\spn}{\mathrm{span}}

\newcommand{\da}{\downarrow}
\newcommand{\ra}{\rightarrow}

\newcommand{\cd}{\cdot}
\newcommand{\ds}{\dots}

\newcommand{\mrm}[1]{\mathrm{#1}}

\newcommand{\diam}{\mathrm{diam}}

\newcommand{\PiH}{\Pi_{\mathrm{H}}}

\newcommand{\cA}{\mathcal{A}}

\newcommand{\cF}{\mathcal{F}}

\newcommand{\cH}{\mathcal{H}}

\newcommand{\cK}{\mathcal{K}}

\newcommand{\cM}{\mathcal{M}}
\newcommand{\cN}{\mathcal{N}}

\newcommand{\cP}{\mathcal{P}}

\newcommand{\cT}{\mathcal{T}}

\newcommand{\cW}{\mathcal{W}}
\newcommand{\cX}{\mathcal{X}}

\newcommand{\cZ}{\mathcal{Z}}

\newcommand{\leb}{\mathrm{Leb}}

\newcommand{\TV}[1]{\left\|#1\right\|_\mathrm{TV}}

\newcommand{\set}[1]{\left\{{#1}\right\}}

\newcommand{\norm}[1]{\left\|#1\right\|}

\newcommand{\norminf}[1]{\left\|#1\right\|_{\infty}}
\newcommand{\normTV}[1]{\left\|#1\right\|_{\mathrm{TV}}}
\newcommand{\abs}[1]{\left|#1\right|}
\newcommand{\sqbk}[1]{\left[ #1 \right]}
\newcommand{\sqbkcond}[2]{\left[ #1 \middle| #2 \right]}
\newcommand{\angbk}[1]{\left\langle #1 \right\rangle}
\newcommand{\crbk}[1]{\left( #1 \right)}

\newcommand{\bmx}[1]{\begin{bmatrix} #1 \end{bmatrix}}

\newcommand{\bd}[1]{\mathbf{#1}}

\newcommand{\argmax}[1]{\underset{#1}{\operatorname{arg} \operatorname{max}}\;}

\title{Q-Measure-Learning for Continuous State RL: \\Efficient Implementation and Convergence }
\author[1]{Shengbo Wang}
\affil[1]{Daniel J. Epstein Department of Industrial and Systems Engineering\\
University of Southern California}

\date{February 2026}

\begin{document}
\maketitle
\begin{abstract}
We study reinforcement learning in infinite-horizon discounted Markov decision processes with continuous state spaces, where data are generated online from a single trajectory under a Markovian behavior policy. To avoid maintaining an infinite-dimensional, function-valued estimate, we propose the novel \emph{Q-Measure-Learning}, which learns a signed empirical measure supported on visited state-action pairs and reconstructs an action-value estimate via kernel integration. The method jointly estimates the stationary distribution of the behavior chain and the Q-measure through coupled stochastic approximation, leading to an efficient weight-based implementation with $O(n)$ memory and $O(n)$ computation cost per iteration. Under uniform ergodicity of the behavior chain, we prove almost sure sup-norm convergence of the induced Q-function to the fixed point of a kernel-smoothed Bellman operator. We also bound the approximation error between this limit and the optimal $Q^*$ as a function of the kernel bandwidth. To assess the performance of our proposed algorithm, we conduct RL experiments in a two-item inventory control setting.
\end{abstract}

\section{Introduction}\label{sec:intro}
We study reinforcement learning (RL) for sequential decision-making arising in engineering settings---such as inventory and revenue management, finance, and learning-based control of physical and robotic systems---where the system state space is most naturally continuous. Even when the primitive state is discrete (e.g., queue lengths in service and manufacturing systems), it is often convenient and effective to embed it into a real vector space, thereby inducing a useful notion of local smoothness for generalization. Motivated by this, we consider the RL of infinite-horizon discounted Markov decision processes (MDPs) with continuous state space $\X\subset\R^{d_\X}$ and continuous or finite action space $\A$. We focus on the single-trajectory setting in which the data stream $\set{(R_{t},X_t,A_t):t\geq 1}$ is generated online by a Markovian behavior policy $\pi_b$. Under standard regularity assumptions, the optimal action-value function $Q^*$ is the unique fixed point of the Bellman optimality operator, which is a contraction in the sup norm; this property underlies the celebrated Q-learning algorithm \cite{watkins1992q}. In continuous state space settings, however, $Q^*$ is an infinite-dimensional object, so tabular Q-learning is not directly applicable without discretization or function approximation. Establishing efficient algorithms and convergence guarantees is therefore subtle, particularly when data is generated by a single trajectory.

At the same time, we observe that simulation under a behavior policy induces a Markov chain $\set{Z_n:n\ge 0}$ on $\Z=\X\times\A$, and many quantities of interest can be consistently estimated by integrals with respect to the empirical measure of the visited locations. This motivates an alternative design principle: rather than approximating $Q^*$ directly in function space, we track an empirical process---the empirical Q-measure---and recover an estimate of $Q^*$ via an integral transform.

Concretely, suppose there exists a finite signed measure $\nu^*$ on $\Z$ and a (hypothetical) \textit{known} smoothing kernel $K:\Z\times\Z\to \R_{>0}$ such that
\begin{equation}\label{eqn:intro_kernel_q}
Q^*(z)\approx q^*(z):=\int_{\Z}K(z,u)\nu^*(du).
\end{equation}
Moreover, suppose $\nu^*$ can be approximated by a reweighted empirical process $\nu_n = \sum_{k=1}^n W_{n,k}\delta_{Z_{n-1}}.$ Then, approximating $Q^*$ reduces to designing an algorithm that iteratively learns $\nu^*$ by updating the weights $\set{W_{n,k}:k=1,\ds,n}$.

Of course, the choice of $K$ is non-trivial, and enabling an efficient implementation requires novel engineering design. We observe that it is hard to specify a \textit{fully known} kernel $K$ and a measure $\nu^*$ that satisfy the above representation. In our setting, however, there is a natural choice of $K$ whose normalization depends only on the stationary distribution $\mu_b$ of the behavior chain $Z$. Since $\mu_b$ can be consistently estimated from the trajectory via its empirical measure, our algorithm updates estimators of $\mu_b$ and $\nu^*$ jointly using stochastic approximation, and then recovers an estimate of $Q^*$ through kernel integration \eqref{eqn:intro_kernel_q}.

This design yields two immediate advantages. First, it enables the use of convergence theory for empirical processes associated with ergodic Markov chains to establish convergence of the empirical Q-measure and, consequently, of the induced Q-function estimate. Second, it leads to memory-efficient implementations: at iteration $n$, it suffices to maintain the historical locations $\set{Z_0,\ds,Z_n}$ together with the associated weights $\set{W_{n,0},\ds,W_{n,n}}$. A further benefit emerges from the efficient algorithmic structure developed in Section~\ref{section:alg_and_implementation}. We design the updates so that, at iteration $n$, the weight vector $\set{W_{n,0},\ds,W_{n,n}}$ can be updated in $O(n)$ time, resulting in an overall procedure that uses $O(n)$ memory and incurs $O(n^2)$ total computation after $n$ iterations.

In this paper, our contributions are organized and presented as follows: 
\begin{itemize}[leftmargin=1.5em]
\item We introduce \emph{Q-Measure-Learning}, an online algorithm that updates a signed measure $\nu_n$ and an empirical reference measure $\mu_n$, reconstructing $q_n$ through a normalized kernel integral.
\item In Section~\ref{section:alg_and_implementation}, we provide an efficient weight-based implementation whose per-iteration computation and memory cost are $O(n)$ at iteration $n$.
\item In Section~\ref{section:convergence}, under uniform ergodicity of the behavior chain, we prove almost sure convergence $q_n\ra q^*$ in sup norm via a Banach-space ODE method, where $q^*$ is the unique fixed-point of a kernel-smoothed Bellman operator.
\item In Section~\ref{section:approx_err}, we quantify the approximation error $\|Q^*-q^*\|$ and show that it can be made arbitrarily small by tuning the choice of the smoothing kernel.
\end{itemize}

\subsection{Related work}

Classical convergence theory for TD methods and Q-learning is well established in finite MDPs, where stochastic approximation and ODE arguments yield almost sure convergence under diminishing stepsizes and sufficient exploration \citep{watkins1992q,jaakkola1993convergence,tsitsiklis1994asynchronous,borkar2000ode}. Recent work establishes finite-sample guarantees for Q-learning, both under access to a generative model and in the single-trajectory setting \citep{wainwright2019stochastic,chen2022finite,li2024ql}. In continuous spaces, a standard approach is to restrict to a function class and analyze the induced approximate/projected Bellman dynamics; early analyses show that convergence can be sensitive to the approximation architecture \citep{tsitsiklis1997td}, while non-expansive/averaging approximations provide a principled route to stability \citep{gordon1995stable,stachurski2008continuous}. Recent work gives finite-iteration, instance-dependent rates for Banach-space stochastic approximation (applicable to Q-learning) \citep{mou2022optimal}, but the resulting methods typically require computing infinite-dimensional objects.

Batch RL methods combine Bellman backups with supervised learning, with representative finite-sample analyses for fitted value iteration and related schemes \citep{ernst2005tree,munos2008fitted}. Least-squares methods provide an alternative batch viewpoint in linear architectures: LSTD solves projected Bellman equations via normal equations for policy evaluation \citep{bradtke1996lstd,boyan2002lstd}, and extensions to action-values and policy iteration include LSQL/LSPI-style algorithms \citep{lagoudakis2002lsmethods,lagoudakis2003lspi}.

Kernel methods smooth rewards/transitions and induce smoothed Bellman operators. KBRL constructs a sample-based empirical MDP via kernel smoothing and then solves Bellman equations on the induced representative set \citep{ormoneit2002kbrl,ormoneit2002avg}; this offline, model-based approach typically has a computation cost that does not scale well with transition data size \citep{jong2006kernel,barreto2016practical}. Related kernelized approaches, such as Gaussian-process RL and kernelized value approximation, also rely on global kernel representations and can incur substantial matrix computation \citep{engel2005gp,taylor2009kernelized}. 

Beyond ``global learning'' approaches, complementary pointwise methods develop multilevel Monte Carlo schemes for Bellman equations that estimate $Q^*$ at a specified state--action pair without constructing a global approximation of $Q^*$ \citep{beck2025nonlinear,meunier2025efficientlearningentropyregularizedmarkov}. These works provide polynomial sample-complexity guarantees in generative-model settings.

In contrast, our method combines the low per-iteration cost and online simplicity of Q-learning with the stability and convergence guarantees typical of kernel-smoothed empirical MDP approaches: it performs a single TD-style update per step while retaining an almost sure sup-norm convergence guarantee to a stationary-normalized, kernel-smoothed Bellman fixed point.
\section{Notations and Assumptions}

Let $\X\subset\R^{d_\X}$ be nonempty and compact, and let $\A$ be either a nonempty compact subset of $\R^{d_\A}$ or a finite set (with the $0$--$1$ metric). While our convergence results extend to compact Polish spaces, the approximation analysis in Section~\ref{section:approx_err} would then require universal kernels on such spaces. To keep the exposition focused on RL, we restrict to Euclidean state spaces. Let $\cX,\cA$ be the Borel $\sigma$-algebra on $\X$ and $\A$ respectively. Define $(\Z,\cZ):=(\X\times\A,\cX\times\cA)$. Let $\cP(\cZ)$ and $\cM(\cZ)$ denote the set of probability measures and signed finite measures on $(\Z,\cZ)$, respectively, and let $C(\Z)$ denote the set of continuous functions $q:\Z\to\R$. We equip $\cM(\cZ)$ and $C(\Z)$ with the total variation norm $\TV{\cd}$ and the sup norm $\norm{\cd}$, respectively.

Fix $\gamma\in(0,1)$. Let $P(\cd| x,a)$ be a Borel controlled transition kernel. Let $\pi_b(\cd|  x)$ be a Markovian behavior policy. Let $(\Omega,\cF,P)$ be a probability space where $\cF$ contains all $P$ null sets supporting a controlled Markov chain $\set{Z_n:n\geq 0}$ so that under $P$
\[
Z_n := (X_n,A_n),\qquad
A_n\sim \pi_b(\cd|  X_n),\qquad
X_{n+1}\sim P(\cd|  X_n,A_n).
\]
Let $P_b$ denote the transition kernel of the Markov chain $\set{Z_n:n\geq 0}$. 

We also consider a sequence of i.i.d. random functions $\set{F_n(\omega,\cd)\in C(\Z):n\ge1}$ on $(\Omega,\cF,P)$ that are independent of $\set{X_n,A_n:n\geq 0}$.  Throughout the paper, we consider the filtration $$\cF_n = \sigma(X_0,A_0,\set{F_k,X_k,A_k:1\leq k\leq n}).$$

\begin{assumption}[Bounded and continuous reward]\label{assump:rwd}
For all $n\ge 0$, $\norm{F_n}\le 1$ a.s. Define the reward $R_{n+1} = F_{n+1}(X_n,A_n,X_{n+1})$ and $r(x,a) = E\sqbk{F_{1}(x,a,X_1)}$. Then $r\in C(\Z)$. 
\end{assumption}

\begin{assumption}[Continuous dynamics]\label{assump:dynamics}
Assume that there exists a probability space $(\W,\cW,\psi)$ where $\W\subset\R^{d_\W}$ is equipped with the Borel $\sigma$-field and a continuous function $f:\Z\times \W\ra \X$ such that for all $B\in\cX$ and $z\in\Z$,
\[
P(B|z)=\int_\W \1\set{f(z,w)\in B}\psi(dw).
\]
\end{assumption}

Under these assumptions, we consider the RL problem of maximizing the infinite-horizon $\gamma$-discounted expected reward:
$$
\sup_{\pi\in\PiH}E^\pi\sqbkcond{\sum_{k=0}^\infty\gamma^k R_{k+1}}{X_0=x},
$$
where $\PiH$ is the class of history-dependent randomized policies, and the expectation $E^{\pi}$ is induced by the policy $\pi$ and the transition kernel $P$.

To achieve RL in this setting, we introduce the Bellman equation for the Q-function. In the classical stochastic control literature, the Bellman optimality operator $\cT$ is defined for $Q\in C(\Z)$ by
\[
\cT[Q](z)
:=
r(z)
+\gamma \int_{\X}\sup_{a'\in\A} Q(x',a')P(dx'| z).
\]
Under Assumptions~\ref{assump:rwd}--\ref{assump:dynamics} and compactness of $\A$, the operator $\cT$ is well-defined and maps $C(\Z)$ into itself--a standard consequence of the maximum theorem and the Feller property; see, e.g., \citet[Assumption~2.1]{stachurski2008continuous}. Since in our setting the reward is bounded by 1, we define the clipped Bellman operator $\overline\cT:C(\Z)\to C(\Z)$ by
\begin{equation}\label{eqn:def_clipped_Bellman}
\overline\cT[Q](x,a)
:= r(x,a)+\gamma\int_{\X}\max_{a'\in\A}\Pi\!\big(Q(x',a')\big)P(dx'| x,a),
\qquad (x,a)\in\Z,
\end{equation}
where $\Pi(x):= \max\set{\min\set{x,\frac{1}{1-\gamma}} ,-\frac{1}{1-\gamma}}$. 

Since both $\cT$ and $\overline\cT$ are a $\gamma$-contractions on $C(\Z)$ and $\norm{r}\le 1$, its unique fixed point $Q^*$ satisfies $\norm{Q^*}\le \frac{1}{1-\gamma}$, hence $\Pi(Q^*)=Q^*$ and therefore $\cT$ and $\overline\cT$ share the same unique fixed point.

Moreover, it is well known that under Assumptions \ref{assump:rwd} and \ref{assump:dynamics}, an optimal deterministic Markovian policy can be obtained by taking greedy action with respect to $Q^*$ (see, e.g. \citet{stachurski2008continuous}). Consequently, learning $Q^*$ from data generated by the behavior policy is of central interest.

\subsection{Q-Measure}
Let $\kappa:\Z\times\Z\ra\R_{>0}$ be a continuous kernel that is lower bounded by $\kappa_\wedge >0$. We will specify $\kappa$ in Section \ref{section:gaussian_kernel}. For any probability measure $\mu\in\cP(\cZ)$, we define the linear operator $\cK_\mu:C(\Z)\ra C(\Z)$ by
\[
\cK_{\mu}[q](z) := \frac{\int_\Z\kappa(z,u)q(u)\mu(du)}{\int_\Z\kappa(z,u)\mu(du)}.
\]
Moreover, for $\mu\in\cP(\cZ)$, we define the linear operator $\Phi_\mu:\cM(\cZ)\ra C(\Z)$ by 
$$\Phi_\mu[\nu](z) := \frac{\int_\Z \kappa(z,u)\nu(du)}{\int_\Z \kappa(z,u)\mu(du)}$$
for all $z\in \Z$. We note that the mapping $\nu\ra \Phi_\mu[\nu]$ will serve as the reconstruction map \eqref{eqn:intro_kernel_q} that produces an estimate of $Q^*$ from $\nu$.

\begin{lemma}\label{lemma:fixed_point_qstar_bound}
For any probability measure $\mu$ on $\cZ$ and any $f,g\in C(\Z)$,
\[
\norm {\cK_\mu [f] - \cK_\mu [g]} \le \norm {f-g}.
\]
Therefore, the smoothed and clipped Bellman operator $\overline\cT_\mu := \cK_{\mu}\circ\overline\cT:C(\Z)\ra C(\Z)$ is a $\gamma$-contraction in $\norm{\cd}$. Hence, it admits a unique fixed point $q_\mu^*\in C(\Z)$ such that $\norm{q^*_\mu}\leq \frac{1}{1-\gamma} $. 
\end{lemma}

\begin{proof}
Fix $z\in\cZ$. Since $\kappa > 0$, we have
\begin{align*}
\abs{\cK_\mu [f](z)-\cK_\mu [g](z)}
\leq
\frac{\int \abs{f(u)-g(u)}\kappa(z,u)\mu(du)}{\abs{\int \kappa(z,u)\mu(du)}}\le
\norm {f-g}.
\end{align*}
Taking $\sup_{z\in\cZ}$ yields the first claim.

Since $\overline\cT$ is a $\gamma$-contraction in $\norm{\cd}$, we have that for all $q,q'\in C(\Z)$,
\begin{align*}\norm{\overline\cT_\mu[q] - \overline\cT_\mu[q']} &= \norm{\cK_\mu[\overline\cT[q]] - \cK_\mu[\overline\cT[q']]} \\
&\leq \norm{\overline\cT[q] - \overline\cT[q']}\\
&\leq \gamma \norm{q-q'};
\end{align*}
i.e. $\overline\cT_\mu$ is a $\gamma$-contraction in $\norm{\cd}$. Since $(C(\Z),\norm{\cd})$ is a Banach space, the Banach fixed-point theorem implies the existence and uniqueness of the fixed point $q_\mu^*$.
\end{proof}

\begin{definition}[The Q-Measure] We define the Q-measure induced by the reference measure $\mu\in\cP(\cZ)$ as 
\begin{equation}\label{eqn:def_Q_measure}
    \nu_\mu^*(A):=\int_{A} \overline\cT [q^*_\mu](z)\mu(dz)
\end{equation}
for all $A\in\cZ$. Note that $\nu_\mu^*\in\cM(\cZ)$ is a finite signed measure.
\end{definition}

The key observation that motivates the construction of our algorithm is that we can represent $q^*_\mu$ using $\nu_\mu^*$
\begin{equation}\label{eqn:nustar_give_qstar}
    \begin{aligned}
        q^*_\mu(z)= \cK_\mu[\overline\cT[q^*_\mu]](z)= \frac{\int_\Z\kappa(z,u)\nu_\mu^*(du)}{\int_{\Z} \kappa(z,u)\mu(du)}= \Phi_\mu[\nu_\mu^*].
    \end{aligned}
\end{equation}
Therefore, if we can learn $\nu_\mu^*$, then we can recover $q_\mu^*$ by applying $\Phi_\mu$; i.e. smooth and normalize $\nu_\mu^*$ by $\kappa$. 

Moreover, when $\cK_\mu$ is close to the identity operator, the fixed point $q^*_\mu$ of
$\overline\cT_\mu=\cK_\mu\circ \overline\cT$ should be close to the fixed point $Q^*$ of $\overline\cT$.
We quantify this approximation error in Section~\ref{section:approx_err}.

\section{Algorithm and Efficient Weight-Based Implementation}\label{section:alg_and_implementation}

The design of our algorithm exploits the ergodicity of the Markov chain $Z$ induced by the behavior policy. In particular, we impose the following assumption.

\begin{assumption}[Uniform ergodicity of the behavior chain]\label{assump:unif_ergodicity}
$(Z_n)_{n\ge 0}$ is uniformly ergodic: there exist $c <\infty $, $\rho\in(0,1)$, and a unique invariant probability measure $\mu_b$ on $\Z$ such that for all $n\ge 0$,
\[
\sup_{z\in\Z}\normTV{P_b^n(z,\cd)-\mu_b(\cd)} \le  c\rho^n.
\]
\end{assumption}

\subsection{Q-Measure-Learning}

We first provide a concrete and intuitive description of our Q-Measure-Learning algorithm. 

We will maintain a probability measure $\mu_n$ and a signed finite measure $\nu_n$ throughout iterations of the algorithm. Initialize $\mu_0:=\delta_{Z_0}$ and $\nu_0:=0$. For $n\ge 0$, define the current iterate approximation for the $q^*$-function as 
\begin{equation}\label{eqn:compute_q}
q_n(z) := \Phi_{\mu_n}[\nu_n](z) = \frac{\int \kappa(z,u)\nu_n(du)}{\int \kappa(z,u)\mu_n(du)},
\end{equation}
$\forall z\in\Z$. Moreover, we choose the stepsize so that the Robbins–Monro condition holds. 
\begin{assumption}\label{assump:stepsize}
    We choose $\set{\alpha_n:n\geq 1}$ s.t. $\alpha_n$ is non-decreasing, $\sum_{n=1}^\infty\alpha_n = \infty$, and $\sum_{n=1}^\infty\alpha_n^2 < \infty$. Moreover, we use $\beta_n = 1/(n+1)$. 
\end{assumption}

Then, conceptually, our algorithm proceeds as follows. 
\begin{enumerate}
\item Sample $(X_{n+1},R_{n+1},A_{n+1})$. 
\item Compute the TD target:
\begin{equation}\label{eqn:compute_y}
Y_{n+1}
:=
R_{n+1} + \gamma \sup_{a\in\A}\Pi\crbk{q_n(X_{n+1},a)}\in\R.
\end{equation}

\item Update $\nu_n$ via stochastic approximation iterate: 
\begin{equation}\label{eqn:nu_update}
\nu_{n+1} := (1-\alpha_{n+1})\nu_n + \alpha_{n+1}Y_{n+1}\delta_{Z_n}.
\end{equation}

\item Update the empirical stationary distribution $\mu_n$:
\begin{equation}\label{eqn:mu_update}
\mu_{n+1} := (1-\beta_{n+1})\mu_n + \beta_{n+1}\delta_{Z_{n+1}}.
\end{equation}

\end{enumerate}
Note that under Assumption \ref{assump:stepsize} and the update \eqref{eqn:mu_update}, $\mu_n$ becomes the empirical measure of $\mu_b$ sampled by the Markov chain $\set{Z_n:n\geq 0}$. 

The update rule \eqref{eqn:nu_update} closely resembles the classical Q-learning algorithm \citep{watkins1992q}. However, instead of directly updating a function value, it augments the empirical measure by assigning weight to the current sample, with magnitude proportional to the one-step Bellman target. Furthermore, the algorithm is constructed so that $\mu_n$, $\nu_n$, $q_n$, and $Y_n$ are all $\cF_n$-measurable.

\subsection{Weight Representation}

\begin{algorithm}[t]
\DontPrintSemicolon
\caption{Q-Measure-Learning for finite $\A$. }\label{alg:Q_meas_learning}
\KwIn{Bandwidth $\sigma>0$, discount factor $\gamma\in(0,1)$, stepsizes $\set{\alpha_n,\beta_n:n\ge 1}$, initial state $X_0$.}

\textbf{Initialize:} $\text{Traj}:=[X_0,A_0]$. \;
Set $\bd{u}\gets [1]$;  $\bd{W}\gets[ ] $.\;

\BlankLine
\For{$n=0,1,2,\dots$}{
Generate $[R_{n+1}, X_{n+1},A_{n+1}]$ and append to the trajectory $\text{Traj}$. \;
    \For{$a\in\A$}{
        $$q(a) \gets
        \dfrac{\sum_{k=1}^{n} \bd W_{k}\kappa((X_{n+1},a),Z_{k-1})}
              {\sum_{k=0}^{n} \bd u_{k}\kappa((X_{n+1},a),Z_{k})}$$
    }
    Compute $Y \gets R_{n+1} + \gamma \max_{a\in\A}\Pi\!\big(q(a)\big)$. \;

    Rescale previous weights: $\bd u \gets (1-\beta_{n+1})\bd u$ and 
        $\bd W \gets (1-\alpha_{n+1})\bd W$.\;
    Append new weights: $\bd u$ appends $\beta_{n+1}$ and $\bd W$ appends $\alpha_{n+1}Y$. 
}
\end{algorithm}

In this section, we demonstrate that the previous abstract version of the Q-Measure-Learning algorithm can be efficiently implemented. Specifically,  we maintain the visited support points $\{Z_0,\dots,Z_{n}\}$ and weights
$\set{u_{n,k}:k=0,\ds,n}$ and $\{W_{n,k}:k=1,\ds ,n\}$ such that
\begin{equation}\label{eqn:weights_rep_mu_nu}
\mu_n = \sum_{k=0}^{n} u_{n,k}\delta_{Z_{k}},
\qquad
\nu_n = \sum_{k=1}^{n} W_{n,k}\delta_{Z_{k-1}}.
\end{equation}
Then, the evaluation for $q_n(X_{n+1},a)$ is just the ratio
$$\begin{aligned}
q_n(X_{n+1},a) = \Phi_{\mu_n}[\nu_n](X_{n+1},a)
= \frac{\sum_{k=1}^{n} W_{n,k}\kappa((X_{n+1},a),Z_{k-1})}
{\sum_{k=0}^{n} u_{n,k}\kappa((X_{n+1},a),Z_{k})}.
\end{aligned}$$

Note that for each $a\in\A$, computing $q_n(X_{n+1},a)$ only need to evaluate $\kappa$ for $n+1$ times and compute two sums. Treating multiplying and summing two numbers and evaluating $\kappa$ as a single operation, evaluating $q_n(X_{n+1},a)$ needs $O(n)$ operations. Hence, if $\A$ is finite, taking maximum and compute $Y_{n+1} = R_{n+1} + \gamma \max_{a\in\A}\Pi\big(q_n(X_{n+1},a)\big)$
can be done with $O(|\A|n)$ operations. 

On the other hand, if $\A$ is continuous, one could apply gradient ascent for $k$ iteration on the reward $q_n(X_{n+1},a_j)$, $j=1,\ds,k$. Then, the maximization procedure will take $O(d_\A   nk)$ CPU time at iteration $n$. 

Next, to update the weights for $\mu_n$ and $\nu_n$, we plug \eqref{eqn:weights_rep_mu_nu} into the measure recursions \eqref{eqn:mu_update} and \eqref{eqn:nu_update}
gives the following weight recursions:
\begin{align*}
u_{n+1,k} &= (1-\beta_{n+1})u_{n,k},\text{ for } k=0,\dots,n; &u_{n+1,n+1} &= \beta_{n+1}. \\
W_{n+1,k} &= (1-\alpha_{n+1})W_{n,k},\text{ for } k=1,\dots,n; &W_{n+1,n+1} &= \alpha_{n+1}Y_{n+1}. 
\end{align*}
This leads to the efficient implementation as described in Algorithm \ref{alg:Q_meas_learning} (we only present the finite action version). In particular, we note that running Algorithm \ref{alg:Q_meas_learning} until iteration $m$, we are only maintaining the variables Traj, $\bd u$ and $\bd{W}$. So, the memory cost is $O(m)$. On the other hand, at the $n$th iterate, the number of operations is $O(|\A|n)$ for finite action settings and $O(d_\A nk)$ for continuous action settings, where the max is approximated by $k$-step gradient ascent. Hence, the total CPU time scales as $O(|\A|m^2)$ and $O(d_\A k m^2)$, respectively.

\section{Convergence}\label{section:convergence}

In this section, we establish the convergence of Q-measure-learning. Specifically, we show that the estimate $q_n$ reconstructed from $\nu_n$ converges to the smooth optimal Q-function $q_{\mu_b}^*$ almost surely (a.s.) in the sup norm. We also show that $\nu_n$ converges to $\nu_{\mu_b}^*$ a.s. in a metric induced by the kernel $\kappa$, which we specify first.

\subsection{Normalized Gaussian-Kernel Metric}\label{section:gaussian_kernel}

Throughout, $\abs{\cd}$ denotes the Euclidean norm on $\R^d$. In this paper, we consider two cases for $\A$, which lead to slightly different choices of the kernel.
\begin{enumerate}[label=(\roman*), leftmargin=2em]
\item If $\A \subset\R^{d_\A}$ is nonempty compact, then $\Z\subset\R^{d_\X + d_\A}$. In this case, we use $\kappa_\sigma(z,u)=\exp\!\big(-\abs{z-u}^2/(2\sigma^2)\big)$.
\item If $\A = \set{a_1,\ds,a_m}$ is finite, for $z=(x,a_i)$ and $u=(y,a_j)$,
\[
\kappa_\sigma(z,u):=\exp\crbk{-\frac{\abs{x-y}^2 + \1_{i\neq j}}{2\sigma^2}}.
\]
\end{enumerate}
We will omit the subscript $\sigma$ when clear.

\begin{definition}\label{def:metric}
For $\kappa = \kappa_\sigma$ as defined earlier, we consider kernel mean distance on $\cM(\cZ)$ by $d(\nu,\nu') := \sup_{z\in\Z}\abs{\int_{\Z} \kappa(z,u)\crbk{\nu-\nu'}(du)}.$ Since $\kappa$ is not normalized, define the stationary-normalized distance
\[
D_{\mu_b}(\nu,\nu')
:=
\sup_{z\in\Z}
\abs{
\frac{\int_{\Z} \kappa(z,u)\crbk{\nu-\nu'}(du)}{\int_{\Z}\kappa(z,u)\mu_b(du)}
}.
\]
Equivalently, $D_{\mu_b}(\nu,\nu')=\norm{\Phi_{\mu_b}(\nu-\nu')}$.
\end{definition}

We first establish that $D_{\mu_b}$ as defined above is a valid metric that can quantify convergence. The proof of Proposition \ref{prop:D_metric} is deferred to Appendix \ref{section:proof:prop:D_metric}. 
\begin{proposition}[$D_{\mu}$ is a metric]\label{prop:D_metric}
In either case (i) or (ii) above, for every $\mu\in\cP(\cZ)$, $D_{\mu}$ is a metric.
\end{proposition}

\begin{remark}
The convergence results below do not rely on the specific Gaussian-type kernel $\kappa$ specified here. We make this choice only to ensure that $D_{\mu_b}$ is a metric, and to obtain concrete approximation error bounds in Section~\ref{section:approx_err}. In particular, one can use other kernels that are bounded away from $0$ on $\Z\times\Z$ and still show that $\norm{q_n-q^*}\ra 0$ a.s.\ as $n\ra\infty$. For other choices of $\kappa$, however, $D_{\mu_b}$ in Definition~\ref{def:metric} may only be a pseudometric (i.e., $D_{\mu_b}(\nu,\nu')=0$ may not imply $\nu=\nu'$). Moreover, the resulting approximation error bounds may differ from those in Theorem~\ref{thm:approx_Q_err}.
\end{remark}

\subsection{Almost sure convergence}

We first establish the almost sure convergence of the empirical distribution $\mu_n$ of the Markov chain $\set{Z_n:n\ge 0}$ to the unique stationary distribution $\mu_b$ in $D_{\mu_b}$. 

\begin{proposition}[Convergence of the empirical process]\label{prop:convergence_of_mu_n}
    Under Assumptions \ref{assump:unif_ergodicity} and \ref{assump:stepsize}, $D_{\mu_b}(\mu_n,\mu_b)\ra 0$ a.s. as $n\ra\infty$. 
\end{proposition}

\begin{proof}

Since $\kappa > 0$ on compact $\Z$, it suffices to show that $d(\mu_n,\mu_b)\to 0$ a.s.

With  $\beta_{n}=1/(n+1)$, the recursion
$\mu_{n}=(1-\beta_{n})\mu_{n-1}+\beta_{n}\delta_{Z_{n}}$
gives $\mu_n=\frac{1}{n+1}\sum_{k=0}^n \delta_{Z_k}$.
Fix $z\in\Z$ and set $h_z(u):=\kappa(z,u)$, which is bounded and measurable. By the
strong law of large numbers for uniformly ergodic Markov chains \citep[Theorem 17.0.1]{meyn2012markov}, $\mu_n[h_z]\ra \mu_b[h_z]$ a.s.

To pass to the supremum over $z$, we use that $\Z$ is compact and $\kappa$ is continuous, 
hence uniformly continuous in the first argument. So, for any $\epsilon>0$, there exists
$\delta>0$ such that $\sup_{u\in\Z}|\kappa(z,u)-\kappa(z',u)|\le \epsilon$ whenever
$|z-z'|\le\delta$. Let $\{z_1,\dots,z_m\}$ be a finite $\delta$-net of $\Z$. Then for any
$z\in\Z$ pick $j$ with $|z-z_j|\le\delta$ and obtain
\[
\big|(\mu_n-\mu_b)[h_z]\big|
\le \big|(\mu_n-\mu_b)[h_{z_j}]\big|
+\mu_n[|h_z-h_{z_j}|]+\mu_b[|h_z-h_{z_j}|]
\le \big|(\mu_n-\mu_b)[h_{z_j}]\big|+2\epsilon.
\]
Taking $\sup_{z\in\Z}$ yields
\[
d(\mu_n,\mu_b)\le \max_{1\le j\le m}\big|(\mu_n-\mu_b)[h_{z_j}]\big|+2\epsilon.
\]
Letting $n\to\infty$ and using the a.s. convergence at
$z_1,\dots,z_m$ gives $\limsup_{n\to\infty} d(\mu_n,\mu_b)\le 2\epsilon$ a.s.
Since $\epsilon>0$ is arbitrary, $d(\mu_n,\mu_b)\to 0$ a.s., and thus
$D_{\mu_b}(\mu_n,\mu_b)\to 0$ a.s.
\end{proof}

Next, we establish the main convergence theorem of this paper. To simplify notation, we let $q^*$ and $\nu^*$ denote $q^*_{\mu_b}$ and $\nu^*_{\mu_b}$, respectively.

\begin{theorem}[Convergence of Algorithm \ref{alg:Q_meas_learning}]\label{thm:as_convergence}
Under Assumptions \ref{assump:rwd}--\ref{assump:stepsize}, $\norm{q_n-q^*}\ra 0$ and $D_{\mu_b}(\nu_n,\nu^*)\ra 0$ a.s. as $n\ra\infty$. 
\end{theorem}

To prove Theorem \ref{thm:as_convergence}, we begin by observing some boundedness properties of objects of interest. 
\begin{lemma}\label{lemma:bounded_trackers}
For all $n\ge 0$, $\mu_n$ is a probability measure; $\abs{Y_{n+1}}\le \frac{1}{1-\gamma}$; $\normTV{\nu_n}\le \frac{1}{1-\gamma}$; $\norm {q_n}\le \frac{1}{(1-\gamma)\kappa_\wedge}$.
\end{lemma}

\begin{proof}[Proof of Lemma \ref{lemma:bounded_trackers}] 
We prove the statements as follows. 

First, $\mu_{n+1}$ is a convex combination of probability measures, hence a probability measure.

Next, by Assumption~\ref{assump:rwd}, $\abs{R_{n+1}}\le 1$.
Also $\abs{\Pi(q_n)}\le \frac{1}{1-\gamma}$, hence $\abs{Y_{n+1}} \le 1 + \frac{\gamma }{1-\gamma} =\frac{1}{1-\gamma}.$

Then, for any bounded measurable $f$ with $\norm {f}\le 1$,
\begin{align*}
\abs{\nu_{n+1}[f]}
&\le
(1-\alpha_{n+1})\abs{\nu_n[f]} + \alpha_{n+1}\abs{Y_{n+1}}\abs{f(Z_n)}\\
&\le
(1-\alpha_{n+1})\normTV{\nu_n} + \alpha_{n+1}\abs{Y_{n+1}}.
\end{align*}
Since $\abs{Y_{n+1}}\le \frac{1}{1-\gamma}$, the recursion implies
$$\normTV{\nu_{n+1}} \le (1-\alpha_{n+1})\normTV{\nu_n} + \frac{\alpha_{n+1}}{1-\gamma}.$$
Since $\nu_0=0$, induction gives $\normTV{\nu_n}\le \frac{1}{1-\gamma}$ for all $n$.

Finally, for any $z\in\Z$,
$$\abs{\int_\Z \kappa(z,u)\nu_n(du)} \le \normTV{\nu_n}\norm {\kappa(z,\cd)} \le \normTV{\nu_n}.$$
Also $\int_\Z \kappa(z,u)\mu_n(du)\ge \kappa_\wedge$.
Hence
$\abs{q_n(z)} \le \frac{\normTV{\nu_n}}{\kappa_\wedge}\le \frac{1}{(1-\gamma)\kappa_\wedge}.$
\end{proof}

\begin{proof}[Proof of Theorem \ref{thm:as_convergence}]
Since $d(\mu_n , \mu_b)\ra 0$, we can consider an auxiliary sequence $\bar q_n$,  the $\mu_b$-normalized $q$-function induced by $\nu_n$; i.e.
\begin{equation}\label{eqn:def_bar_q}
\bar q_n(z) := \Phi_{\mu_b}[\nu_n](z) = \frac{\int \kappa(z,u)\nu_n(du)}{\int \kappa(z,u)\mu_b(du)}.
\end{equation}

\begin{lemma}\label{lemma:q_bar_approx_q}
$\norm {q_n-\bar q_n}
\leq \frac{D_{\mu_b}(\mu_n,\mu_b)}{(1-\gamma)\kappa_\wedge} \ra 0$ a.s.
\end{lemma}

\begin{proof}[Proof of Lemma \ref{lemma:q_bar_approx_q}]
For any $z\in\Z$,
\[
q_n(z)-\bar q_n(z)
=
\int \kappa(z,u)\nu_n(du)\crbk{\frac{1}{\int_\Z \kappa(z,u)\mu_n(du)}-\frac{1}{\int_\Z \kappa(z,u)\mu_b(du)}}.
\]
For simplicity, let 
\begin{equation}\label{eqn:def_c_mu}
    c_\mu(z):=\int_\Z \kappa(z,u)\mu(du).
\end{equation} By Lemma~\ref{lemma:bounded_trackers}, $\abs{\int \kappa(z,u)\nu_n(du)}\le \normTV{\nu_n}\le \frac{1}{1-\gamma}$.
Also $c_{\mu_n}(z)\ge \kappa_\wedge$ and $c_{\mu_b}(z)\ge \kappa_\wedge$, hence
\[
\abs{\frac{1}{c_{\mu_n}(z)}-\frac{1}{c_{\mu_b}(z)}}
=
\frac{\abs{c_{\mu_n}(z)-c_{\mu_b}(z)}}{c_{\mu_n}(z)c_{\mu_b}(z)}
\le
\frac{D_{\mu_b}(\mu_n,\mu_b)}{\kappa_\wedge}.
\]
Therefore
\[
\norm {q_n-\bar q_n}
\le
\frac{D_{\mu_b}(\mu_n,\mu_b)}{(1-\gamma)\kappa_\wedge}.
\]
Proposition~\ref{prop:convergence_of_mu_n} gives $D_{\mu_b}(\mu_n,\mu_b)\to 0$ a.s., hence the claimed result holds.
\end{proof}

Therefore, to show $\norm {q_n-q^*}\ra 0$ a.s., it suffices to show the a.s. convergence of $\norm{\bar q_n- q^*}\ra 0$. To achieve this, we analyze the algorithm's recursion. From the recursion \eqref{eqn:nu_update} and the definition of $\bar q_n$ in \eqref{eqn:def_bar_q}, 
\begin{equation}\label{eqn:tu_Delta_initial_split}
    \begin{aligned}
        \bar q_{n+1} 
        &= \Phi_{\mu_b}[\nu_{n+1} ]\\
        &= \Phi_{\mu_b}[(1-\alpha_{n+1})\nu_{n} ] + \alpha_{n+1}Y_{n+1}\Phi_{\mu_b}[\delta_{Z_n}]\\
        &= \bar q_{n} - \alpha_{n+1} \bar q_n + \alpha_{n+1}Y_{n+1} \Phi_{\mu_b}[\delta_{Z_n}]
    \end{aligned}
\end{equation}

To connect $Y_{n+1}$ with the Bellman equation, we consider the Martingale difference 
\begin{equation}\label{eqn:def_D_n}
    D_{n+1}:= \crbk{Y_{n+1} - E[Y_{n+1}|\cF_n] }\Phi_{\mu_b}[\delta_{Z_n}]
\end{equation}
where, by the Markov property
$$E[Y_{n+1}|\cF_n] = r(X_n,A_n) + \gamma\int \sup_{a\in\A}\Pi(q_n(x',a))P(dx'|X_n,A_n)= \overline\cT[q_n](Z_n).$$

Therefore, letting

\begin{equation}\label{eqn:def_gn_bn}
\begin{aligned}
G_n(z)
&:= \overline\cT[\bar q_n](Z_n)\Phi_{\mu_b}[\delta_{Z_n}](z) - \overline\cT_{\mu_b}[\bar q_n](z)\\
B_n(z)
&:= \overline\cT[q_n](Z_n) \Phi_{\mu_b}[\delta_{Z_n}](z) - \overline\cT[\bar q_n](Z_n)\Phi_{\mu_b}[\delta_{Z_n}](z)
\end{aligned}
\end{equation}
for all $z\in\Z$, we can rewrite \eqref{eqn:tu_Delta_initial_split} as
\begin{equation}\label{eqn:tu_Delta_split}
    \begin{aligned}
        \bar q_{n+1}
        &= \bar q_{n} - \alpha_{n+1} \bar q_{n} + \alpha_{n+1}Y_{n+1} \Phi_{\mu_b}[\delta_{Z_n}]\\
        &= \bar q_{n} + \alpha_{n+1} \crbk{\overline\cT_{\mu_b} [\bar q_n]-\bar q_n
        } - \alpha_{n+1} \overline\cT_{\mu_b} [\bar q_n]  + \alpha_{n+1}D_{n+1}  + \alpha_{n+1}\overline \cT[q_n] (Z_n)\Phi_{\mu_b}[\delta_{Z_n}]\\
        &=  \bar q_{n} + \alpha_{n+1} \crbk{\overline\cT_{\mu_b} [\bar q_n]-\bar q_n
        }  + \alpha_{n+1}D_{n+1} + \alpha_{n+1}\crbk{\overline \cT[q_n](Z_n) \Phi_{\mu_b}[\delta_{Z_n}]-  \overline\cT_{\mu_b} [\bar q_n]} \\
        &=  \bar q_{n} + \alpha_{n+1} \crbk{\overline\cT_{\mu_b} [\bar q_n]-\bar q_n
        }  + \alpha_{n+1}(D_{n+1} +G_n + B_n).
    \end{aligned}
\end{equation}

We note that $\set{D_n:n\geq1}$, $\set{G_n:n\geq 0}$, and $\set{B_n:n\geq 0}$ are interpreted as the martingale difference noise, Markov noise, and bias sequence, respectively. Our proof relies on analyzing these terms separately and obtaining the following result. 
\begin{proposition}[Vanishing uniform error over constant time window] \label{prop:vanishing_err}Let $$\cN(n,T):=\set{m\geq n:\sum_{k=n}^{m-1}\alpha_{k+1}\leq T}. $$ Then, for any fixed $T\geq 0$, 
\begin{equation}\label{eqn:uniform_noise_control}
    \max_{m\in \cN(n,T)}\norm{\sum_{k=n}^{m-1}\alpha_{k+1}(D_{k+1 }+G_{k} + B_k)}\ra 0 \quad \text{a.s.}
\end{equation}
\end{proposition}
To streamline the proof of Theorem \ref{thm:as_convergence}, we defer the proof of Proposition \ref{prop:vanishing_err} to Appendix \ref{section:proof:prop:vanishing_err}.

Equipped with this error bound, we proceed to show the convergence by using the classic ODE approach (c.f. \citet{kushner2003stochastic}). We note that since we are working with Banach-space valued objects $\bar q_n\in C(\Z)$, additional techniques beyond the vector valued case in \citet{kushner2003stochastic} are required. 

To simplify notation, we define $\cH[q] := \overline\cT_{\mu_b} [q]-q$. Notice that since $\overline\cT_{\mu_b}$ is a $\gamma$-contraction, for all $q,q'\in C(\Z)$, \begin{equation}\label{eqn:cH_Lip}
    \norm{\cH[q] - \cH[q']}\leq \norm{\overline\cT_{\mu_b} [q]-\overline\cT_{\mu_b}[q']}+\norm{q-q'}\leq (1+\gamma)\norm{q-q'};
\end{equation}
i.e. $\cH$ is $(1+\gamma)$-Lipschitz. 

With this notation, the recursion of $\bar q_n$ in \eqref{eqn:tu_Delta_split} implies that
$$\bar q_m = \bar q_n + \sum_{k=n}^{m-1}\alpha_{k+1}\cH[q_k] + \sum_{k=n}^{m-1}\alpha_{k+1}(D_{k+1} + G_k + B_k).$$

We first state the (Banach-space valued) ODE of interest and the relevant existence and uniqueness of solution theorem in Lemma \ref{lemma:ODE_EU_solution}. Then, we establish the asymptotic stability of the fixed-point of this ODE in Lemma \ref{lemma:ODE_exponential_stability_of_qstar}. 

\begin{lemma}[\citet{brezis2011functional}, Theorem 7.3]\label{lemma:ODE_EU_solution}
Let $(E,\norm{\cd})$ be a Banach space and let $\cH:E\to E$ be globally Lipschitz; i.e., there exists $\ell <\infty$ such that for all $ q,q'\in E$, $\norm{\cH[q]-\cH[q']}\le \ell\norm{q-q'}$.
Then for every $q_0\in E$ there exists a unique function $q:[0,\infty)\ra E$ such that $q$ is continuously differentiable and for all $t\geq 0$
\begin{equation}\label{eqn:Banach_integral_equation}
q(t)=q_0+\int_0^t \cH(q(s))ds. 
\end{equation}
\end{lemma}

Apply Lemma \ref{lemma:ODE_EU_solution} with $(E,\norm\cd) =  (C(\Z),\norm{\cd})$, we define $q(t;q')$ as the unique solution to
\begin{equation}\label{eqn:ode_q_def}
    q(t;q') := q' + \int_0^{t}\cH[q(s;q')]ds.
\end{equation} 

\begin{lemma}[Global exponential stability of $q^*$]\label{lemma:ODE_exponential_stability_of_qstar}
Recall that $q^* = q^*_{\mu_b}$ is the unique fixed point of $\overline\cT_{\mu_b}$ in $C(\Z)$. For all $t\geq 0$ and $q'\in C(\Z)$ 
$$\norm{q(t,q') - q^*} \leq e^{-(1-\gamma)t}\norm{q' - q^*}.$$
\end{lemma}

\begin{proof}[Proof of Lemma \ref{lemma:ODE_exponential_stability_of_qstar}]
Note that when $q' = q^*$ the constant function $t\ra q^*$ solves \eqref{eqn:ode_q_def}. Then
\[
\begin{aligned}
\epsilon(t)&:= q(t;q') - q^* \\
&= (q'-q^*) + \int_0^t\crbk{ \cH[q(s;q')]-\cH[q^*]}ds.\\
&= \epsilon(0) + \int_0^t \crbk{ \overline \cT_{\mu_b}[q(s;q')]-\overline \cT_{\mu_b}[q^*] + \epsilon(s)}ds.
\end{aligned}
\]
Therefore, using integration by parts, we see that
\[
e^t\epsilon(t) = \epsilon(0) + \int_0^te^s\crbk{ \overline \cT_{\mu_b}[q(s;q')]-\overline \cT_{\mu_b}[q^*] }ds.
\]
Taking the sup norm, we have
\[
e^{t}\norm{\epsilon(t)}\leq \norm{\epsilon(0)} + \int_0^t e^{s}\gamma\norm{q(s;q') - q^*}ds = \norm{\epsilon(0)} + \int_0^t e^{s} \gamma\norm{\epsilon(s)}ds
\]
where we used the contraction property of $ \overline \cT_{\mu_b}$. 

Therefore, letting $x(t) = e^t \norm{\epsilon(t)}$, we see that $$x(t)\leq x(0) + \gamma\int_0^t   x(s)ds.$$
By Gr\"onwall's inequality, $x(t)\le \norm{x(0)}e^{\gamma t} = \norm{\epsilon(0)}e^{\gamma t}$, hence $$\norm{\epsilon(t)}\le e^{-t}x(t)\le e^{-(1-\gamma)t}\norm{\epsilon(0)}.$$ This implies Lemma \ref{lemma:ODE_exponential_stability_of_qstar}. 
\end{proof}

Define the continuous time in between $n$ and $m$ by $$ \tau_{n,m}:= \sum_{k=n}^{m-1}\alpha_{k+1}. $$ So, we have that for all $k\geq n$
\begin{equation}\label{eqn:tu_max_nk_property}
\max \cN(n,\tau_{n,k}) = k \quad \text{and hence} \quad \tau_{n,\max\cN(n,T)} \leq T.
\end{equation}

We consider the error between the discrete update and the ODE flow at time $\tau_{n,m}$. Specifically, we subtract and rewrite the error as
\begin{equation}\label{eqn:tu_q_and_ODE_error_split}
    \begin{aligned}
        \bar q_m - q(\tau_{n,m};\bar q_n) &= \underbrace{\sum_{k=n}^{m-1} \alpha_{k+1}\crbk{\cH[\bar q_k] - \cH[q(\tau_{n,k};\bar q_n)]}}_{:=\psi_{n,m}} \\
        &\quad + \underbrace{\sum_{k=n}^{m-1} \crbk{\alpha_{k+1}\cH[q(\tau_{n,k};\bar q_n)]-\int_{\tau_{n,k}}^{\tau_{n,k+1}} \cH[q(s;\bar q_n)]ds }}_{:=\phi_{n,m}}\\
        &\quad+\underbrace{\sum_{k=n}^{m-1} \alpha_{k+1}\crbk{D_{k+1} + G_k+B_k }}_{:=\veps_{n,m}}
    \end{aligned}
\end{equation}

We analyze these terms separately. First, for $\psi_{n,m}$, we have
$$\begin{aligned}\norm{\psi_{n,m}}&\leq \sum_{k=n}^{m-1}\alpha_{k+1}\crbk{\norm{\overline\cT_{\mu_b}[\bar q_k] - \overline\cT_{\mu_b}[q(\tau_{n,k};\bar q_n)]} + \norm{\bar q_k- q(\tau_{n,k};\bar q_n)} }\\
&\leq (1+\gamma)\sum_{k=n}^{m-1}\alpha_{k+1}\norm{\bar q_k-q(\tau_{n,k};\bar q_n)}
\end{aligned}$$
where we used the contraction property of $\overline\cT_{\mu_b}$. 

For $\phi_{n,m}$, note that $\tau_{n,k+1} - \tau_{n,k} = \alpha_{k+1}$. So, 
\begin{equation}\label{eqn:tu_phi_n_m_bd_step}\begin{aligned}\norm{\phi_{n,m}}&\leq \sum_{k=n}^{m-1}\int_{\tau_{n,k}}^{\tau_{n,k+1}} \norm{\cH[q(\tau_{n,k};\bar q_n)] - \cH[q(s;\bar q_n)]}ds\\
&\leq (1+\gamma) \sum_{k=n}^{m-1}\int_{\tau_{n,k}}^{\tau_{n,k+1}} \norm{q(\tau_{n,k};\bar q_n)- q(s;\bar q_n)}ds\\
&\stackrel{(i)}= (1+\gamma) \sum_{k=n}^{m-1}\int_{\tau_{n,k}}^{\tau_{n,k+1}} \norm{\int_{\tau_{n,k}}^s \cH[q(t;\bar q_n)] dt}ds\\
&\leq \frac{(1+\gamma)}{2}\sum_{k=n}^{m-1}\sup_{\tau_{n,k}\leq s\leq \tau_{n,k+1}}\norm{\cH[q(s;\bar q_n)]} \alpha_{k+1}^2\\
&\leq\sup_{0\leq s\leq \tau_{n,m}}\norm{\cH[q(s;\bar q_n)]} \sum_{k=n}^{m-1}\alpha_{k+1}^2
\end{aligned}
\end{equation}
where $(i)$ follows from \eqref{eqn:ode_q_def} and the last equality uses $\gamma < 1$. 

By Lemma \ref{lemma:bounded_trackers} and $\kappa\leq 1$, we have that $$\bar q_n(z)= \frac{\int_{\Z}\kappa(z,u)\nu_n(du)}{\int_{\Z}\kappa(z,u)\mu_b(du)}\leq \frac{\TV{\nu_n}} {\kappa_\wedge}\leq \frac{1}{(1-\gamma)\kappa_\wedge}.$$
So, by Lemma \ref{lemma:ODE_exponential_stability_of_qstar}, we have that
$$\norm{q(s;\bar q_n) - q^*}\leq e^{-(1-\gamma)t} \norm{\bar q_n-q^*}\leq \frac{2}{(1-\gamma)\kappa_\wedge}.$$ Thus, from \eqref{eqn:tu_phi_n_m_bd_step}, we conclude that
\begin{equation}\label{eqn:tu_phi_n_m_bd}\norm{\phi_{n,m}}\leq \frac{2}{(1-\gamma)\kappa_\wedge}\sum_{k=n}^{m-1}\alpha_{k+1}^2. \end{equation}
Since $\sum_{k=0}^\infty\alpha_{k+1}^2 < \infty$, we have that as $n\ra\infty,$
\begin{equation}\label{eqn:tu_phi_unif_interval_bd}\max_{m\in\cN(n,T)}\norm{\phi_{n,m}}\ra 0
\end{equation}everywhere.

Therefore, going back to \eqref{eqn:tu_q_and_ODE_error_split}, we have
\begin{equation}\label{eqn:tu_q_and_ODE_error_norm_split}
\begin{aligned}
    \norm{\bar q_m - q(\tau_{n,m};\bar q_n)}&\leq (1+\gamma)\sum_{k=n}^{m-1}\alpha_{k+1}\norm{\bar q_k-q(\tau_{n,k};\bar q_n)} +\norm{\phi_{n,m}} + \norm{\veps_{n,m}}.
\end{aligned}
\end{equation}
We now strengthen this bound to a uniform version over all time within a constant interval as $n\ra\infty$. In particular, we look at the uniform error defined by $$\Delta_{n}(T):= \max_{m\in\cN(n,T)}\norm{\bar q_m - q(\tau_{n,m};\bar q_n)}.$$

Notice that for all $k\geq n$  and $t\in [\tau_{n,k},\tau_{n,k+1})$ $$\Delta_{n}(t)=\max_{m\in\cN(n,t)}\norm{\bar q_m - q(\tau_{n,m};\bar q_n)} \geq 
\norm{\bar q_k - q(\tau_{n,k};\bar q_n)}.$$
So, $$\int_{\tau_{n,k}}^{\tau_{n,k+1}}\Delta_{n}(t)dt \geq \alpha_{k+1}\norm{\bar q_k - q(\tau_{n,k};\bar q_n)}.$$ Combining these properties with \eqref{eqn:tu_q_and_ODE_error_norm_split}, we have that
\[\begin{aligned}\Delta_n(T)&\leq (1+\gamma)\sum_{k=n}^{\max \cN(n,T)-1}\int_{\tau_{n,k}}^{\tau_{n,k+1}}\Delta_n(t)dt + \max_{m\in\cN(n,m)} (\norm{{\phi_{n,m}}} + \norm{\veps_{n,m}})\\
    &\leq (1+\gamma)\int_{0}^{T}\Delta_n(t)dt + \max_{m\in\cN(n,T)} (\norm{{\phi_{n,m}}} + \norm{\veps_{n,m}})
\end{aligned}\]
where the last inequality follows from $\tau_{n,n} = 0$ and \eqref{eqn:tu_max_nk_property}. Therefore, applying Gr\"onwall's inequality yields
\begin{equation}\label{eqn:tu_Delta_n_T_Gronwall}
    \Delta_n(T)\leq e^{(1+\gamma)T}\crbk{\max_{m\in\cN(n,T)} \norm{{\phi_{n,m}}} + \max_{m\in\cN(n,T)}\norm{\veps_{n,m}}}. 
\end{equation}
Therefore, recalling the definition of $\veps_{n,m}$ and applying Proposition \ref{prop:vanishing_err} and \eqref{eqn:tu_phi_unif_interval_bd}, we conclude that for fixed $T > 0$, 
$\Delta_n(T)\ra 0$ a.s. as $n\ra\infty$. 

We now use the a.s. convergence of $\Delta_n(T)$ together the stability of $q^*$  in Lemma \ref{lemma:ODE_exponential_stability_of_qstar} to conclude $\bar q_n\to q^*$ a.s. Specifically, fix $T>0$ we recursively define
$m_0 := 0$ and
$$m_{k+1} := \max\cN(m_k,T),\quad k\geq 0.$$
By \eqref{eqn:tu_max_nk_property}, $\tau_{m_k,m_{k+1}}\leq T. $ Moreover, since $\alpha_k\ra 0,$ it is easy to see that $\tau_{m_k,m_{k+1}}\ra T$ as $k\ra\infty$. In particular, there exists $k_T$ sufficiently large so that for all $k\geq k_T$, $\tau_{m_k,m_{k+1}}\geq T/2$.

We first control the error at iteration $\set{m_k:k\geq 0}$. Let $E_{k+1}:=\norm{\bar q_{m_{k+1}} - q^*}$, then
\[
\begin{aligned}
    E_{k+1} &\leq \norm{\bar q_{m_{k+1}} - q(\tau_{m_k,m_{k+1}};\bar q_{m_k})} + \norm{q(\tau_{m_k,m_{k+1}};\bar q_{m_k}) - q^*}\\
    &\stackrel{(i)}\leq \Delta_{m_k}(T) + \exp(-(1-\gamma)\tau_{m_k,m_{k+1}})\norm{\bar q_{m_k} - q^*}\\
    &= \Delta_{m_k}(T) + \exp(-(1-\gamma)\tau_{m_k,m_{k+1}})E_k
\end{aligned}
\]
where $(i)$ used the definition of $m_{k+1}$ so that $m_{k+1}\in\cN(m_k,T)$, and Lemma \ref{lemma:ODE_exponential_stability_of_qstar}. Therefore for $k\ge k_T$
\begin{equation}\label{eqn:tu_mk_q_err_bd}
E_{k+1} \leq e^{-\frac{(1-\gamma)T}{2}}E_k + \Delta_{m_k}(T) 
\end{equation}

Since $\Delta_n(T)\ra 0$ a.s., for a.s.$\omega\in\Omega$ and any $\epsilon > 0$, there exists $K_T(\epsilon,\omega)\geq k_T$ s.t. for all $k\geq K_T(\epsilon,\omega)$, $$\Delta_{m_k}(T,\omega)\leq (1-e^{-(1-\gamma)T/2}) \epsilon. $$
Therefore, for such $\omega$ and all $k\geq K_T(\epsilon,\omega)$, $$E_{k+1}(\omega)\leq e^{-(1-\gamma)T/2} E_k(\omega) + (1-e^{-(1-\gamma)T/2})\epsilon.$$
It is not hard to see that 
\begin{align*}E_{n} &\leq  e^{-(1-\gamma)T(n-k)/2}E_k + (1-e^{-(1-\gamma)T/2})\sum_{j=0}^{n-k-1}e^{-(1-\gamma)Tj/2}\epsilon \\
&\leq e^{-(1-\gamma)T(n-k)/2}E_k + \epsilon   
\end{align*}
for all $n\geq k\geq K_T(\epsilon,\omega)$. Since $\epsilon$ is arbitrary, we conclude that $E_n\ra 0$ a.s. 

Finally, for any $m\in\cN(m_k,T)$, we have that
\begin{align*}
\norm{\bar q_m - q^*} &\leq \norm{\bar q_m-q(\tau_{m_k,m},\bar q_{m_k})} + \norm{q(\tau_{m_k,m},\bar q_{m_k}) - q^*} \\
&\stackrel{(i)}\leq \Delta_{m_k}(\tau_{m_k,m_{k+1}}) + \norm{\bar q_{m_k} - q^*}\\
&\stackrel{(ii)}\leq \Delta_{m_k}(T) + E_k
\end{align*}
where $(i)$ applies max over $m = m_k,\ds,m_{k+1}$ and Lemma \ref{lemma:ODE_exponential_stability_of_qstar}, and $(ii)$ uses the monotonicity of $\Delta_{m_k}(\cd)$ and \eqref{eqn:tu_max_nk_property}. Since $\sum_{k}\alpha_{k+1} = \infty$, $m_k\ra\infty$ as $k\ra\infty$, we have 
$$ \limsup_{k\ra\infty}\max_{m\in\cN(m_k,T)}\norm{\bar q_m - q^*} = \limsup_{k\ra\infty}[\Delta_{m_k}(T) + E_k] = 0 \quad \text{a.s.}$$
Therefore, a.s.
$$0 = \lim_{k\ra\infty}\sup_{n\ge k}\max_{m\in\cN(m_n,T)}\norm{\bar q_m - q^*} = \lim_{k\ra\infty}\sup_{n\geq m_k}\norm{\bar q_n - q^*},$$
which implies that $\norm{\bar q_n - q^*}\ra 0$ and hence $\norm{ q_n - q^*}\ra 0$ a.s. by Lemma \ref{lemma:q_bar_approx_q}.

Finally, note that $\norm{\bar q_n - q^*} = \norm{\Phi_{\mu_b}[\nu_n - \nu^*]} = D_{\mu_b}(\nu_n,\nu^*).$
Therefore, $D_{\mu_b}(\nu_n,\nu^*)\ra 0$ a.s. 
\end{proof}

\section[Approximating Qstar with qstar]{Approximating $Q^*$ with $q^*$}\label{section:approx_err}

We have established that $q_n$ consistently approximates $q^*$ as $n\ra\infty$. We now quantify the bias from smoothing by bounding the deterministic approximation error between the optimal Q-function $Q^*$ (the fixed point of $\overline\cT$) and the stationary-smoothed fixed point $q^*=q^*_{\mu_b}$ (the fixed point of $\overline\cT_{\mu_b}$).

To simplify notation, we use $\mathsf d$ to denote the following metric
\[
\mathsf d(z,u):=
\begin{cases}
\abs{z-u}, & \text{case (i): if $\A\subset\R^{d_\A}$ is compact; }\\
\sqrt{\abs{x-y}^2+\1_{\{i\neq j\}}}, & \text{case (ii): if $\A=\{a_1,\dots,a_m\}$}. 
\end{cases}
\]
In either cases, we consider a reference measure 
\[
\lambda :=
\begin{cases}
\leb \text{ on }\Z\subset\R^{d_\X+d_\A}, & \text{case (i)},\\
\leb \times \text{(counting measure on $\A$)} \text{ on }\Z=\X\times\A, & \text{case (ii)},
\end{cases}
\]
where $\leb$ denotes the Lebesgue measure. Also, we write the volume of a Euclidean ball in $\R^d$ as 
\[
V_d(r): = \leb\big(\{v\in\R^{d}:\abs{v}\le r\}\big)=\frac{\pi^{d/2}r^{d}}{\Gamma(d/2+1)}. 
\]

We show that the following approximation error bounds hold. 

\begin{theorem}[Approximation error of smoothed Q-function]\label{thm:approx_Q_err}
Assume that $Q^*$ is $\alpha$-H\"older in $\mathsf d$ for some $\alpha\in(0,1]$, i.e.\ there exists $L_Q<\infty$ such that
$\abs{Q^*(z)-Q^*(u)}\le L_Q\,\mathsf d(z,u)^{\alpha}$ for all $z,u\in\Z$. Then, for any $\sigma>0$,
\begin{equation}\label{eqn:bias_via_xi}
\norm{Q^*-q^*}
\le \frac{L_Q}{1-\gamma}\,\xi_{\mu_b}(\sigma),
\end{equation}
where
\begin{equation}\label{eqn:def_xi_sigma_both}
\xi_{\mu_b}(\sigma)
:=
\sup_{z\in\Z}
\frac{\int_{\Z}\mathsf d(z,u)^{\alpha}\,\kappa_\sigma(z,u)\,\mu_b(du)}
{\int_{\Z}\kappa_\sigma(z,u)\,\mu_b(du)}.
\end{equation}

In addition, assume that $\mu_b\ll \lambda$ admits a density $p_b=\frac{d\mu_b}{d\lambda}$ such that $0<\underline p_b = \essinf_\lambda p_b\le \overline p_b = \esssup_\lambda p_b<\infty$. Assume also the following uniform local volume condition: There exist constants $0 < v_\wedge\leq 1$ and $\sigma_\vee>0$ such that for all $z\in\Z$

\begin{equation}\label{eqn:local_volume_ball}
\begin{aligned}
&\text{case (i):}
&&\leb\big(\{u\in\Z:\abs{z-u}\le r\}\big)\;\ge\; v_\wedge\,V_{d_\X+d_\A}(r), &&\forall r\in (0,\sigma_\vee\sqrt{d_{\X}+d_\A}];\\
&\text{case (ii):} &&\leb\big(\{y\in\X:\abs{x-y}\le r\}\big)\;\ge\; v_\wedge\,V_{d_\X}(r),&&\forall r\in (0,\sigma_\vee\sqrt{d_{\X}}].
\end{aligned}
\end{equation}
Then, for all $\sigma\in(0,\sigma_\vee]$,
\begin{equation}\label{eqn:xi_rate}
\begin{aligned}
&\text{case (i):}
&&\xi_{\mu_b}(\sigma)\le \frac{2}{v_\wedge}\frac{\overline p_b}{\underline p_b}\,(d_\X+d_\A)^{\alpha/2}\,\sigma^{\alpha},\\
&\text{case (ii):} &&\xi_{\mu_b}(\sigma)\le\frac{2}{v_\wedge}\frac{\overline p_b}{\underline p_b}
\Big(d_\X^{\alpha/2} \sigma^{\alpha}+(m-1)(\diam(\X)+1)^{\alpha}e^{-1/(2\sigma^2)}\Big).
\end{aligned}
\end{equation}
\end{theorem}

\begin{remark}
The optimal Q-function is H\"older continuous when the dynamics $f$ in Assumption~\ref{assump:dynamics} is uniformly Lipschitz on $\Z$ and the reward is H\"older continuous; see, e.g., \citet{harder2024continuity}.

Note that it is necessary that $v_\wedge\leq 1$ by the definition of $V_d(r)$. Further, we remark that the local volume condition rules out cases where $\X\times \A$ (in case (i)) or $\X$ (in case (ii)) is ``hollow" near any point $z\in\Z$. In particular, if $\Z/\X$ is a convex body with non-empty interior, then the local volume condition is satisfied. Therefore, the approximation error goes to 0 at rate $\lesssim \sigma^\alpha$ as $\sigma\da 0$ if $\mu_b$ has a bounded density and $\Z$ is, e.g., a full-dimensional sphere or rectangle. 

We also remark that the technical requirement that the local volume condition holds up to radii $r\le O(\sqrt d)$ arises from matching spherical and Gaussian volume in $\R^d$. 
\end{remark}

\begin{proof}[Proof of Theorem \ref{thm:approx_Q_err}]
We first assume only that $Q^*$ is $\alpha$-H\"older in $\mathsf d$. Using fixed-point identities
$Q^*=\overline\cT[Q^*]$ and $q^*=q^*_{\mu_b}=\cK_{\mu_b}[\overline\cT[q^*]]$,
\[
Q^*-q^*
=\overline\cT[Q^*]-\cK_{\mu_b}[\overline\cT[q^*]]
=\big(Q^*-\cK_{\mu_b}[Q^*]\big)+\cK_{\mu_b}\!\big(\overline\cT[Q^*]-\overline\cT[q^*]\big).
\]
Taking $\norm{\cd}$, using that $\cK_{\mu_b}$ is non-expansive (Lemma~\ref{lemma:fixed_point_qstar_bound})
and $\overline\cT$ is a $\gamma$-contraction, yields
\[
\norm{Q^*-q^*}
\le \norm{Q^*-\cK_{\mu_b}[Q^*]}+\gamma\norm{Q^*-q^*},
\]
hence $\norm{Q^*-q^*}\le \frac{1}{1-\gamma}\norm{Q^*-\cK_{\mu_b}[Q^*]}$.

Fix $z\in\Z$ and define the probability measure
\[
w_{z,\sigma}(A):=\frac{\int_A\kappa_\sigma(z,u)\mu_b(du)}{\int_\Z \kappa_\sigma(z,v)\mu_b(dv)}.
\]
Then $\cK_{\mu_b}[Q^*](z)=\int_\Z Q^*(u)w_{z,\sigma}(du)$, and by H\"older regularity,
\[
\abs{Q^*(z)-\cK_{\mu_b}[Q^*](z)}
\le \int_\Z \abs{Q^*(z)-Q^*(u)}\,w_{z,\sigma}(du)
\le L_Q \int_\Z \mathsf d(z,u)^\alpha\,w_{z,\sigma}(du).
\]
Taking $\sup_{z\in\Z}$ gives $\norm{Q^*-\cK_{\mu_b}[Q^*]}\le L_Q\,\xi_{\mu_b}(\sigma)$, and thus
\eqref{eqn:bias_via_xi} follows.

\medskip
We now prove \eqref{eqn:xi_rate} under the additional density and local volume assumptions. For case (i), set $d:=d_\X+d_\A$ and for case (ii), set $d:=d_\X$.

\textbf{Denominator lower bound: cases (i).}
Fix $\sigma\in(0,\sigma_\vee]$, let $R:=\sigma\sqrt{d}$. For any $z\in\Z$,
\[
c_{\mu_b}(z)=\int_\Z \kappa_\sigma(z,u)\mu_b(du)
=\int_\Z e^{-\abs{z-u}^2/(2\sigma^2)}\,p_b(u)\,du.
\]
Using $p_b\ge \underline p_b$ and restricting to $B(z,R):=\{u\in\Z:\abs{z-u}\le R\}$ gives
\begin{equation}\label{eqn:tu_c_lb_init}
c_{\mu_b}(z)\ge \underline p_b\int_{B(z,R)} e^{-\abs{z-u}^2/(2\sigma^2)}\,du.    
\end{equation}

Define $\phi(r):=e^{-r^2/(2\sigma^2)}\in(0,1]$ . We observe that for all $z\in\Z$,
\begin{equation}\label{eqn:tu_int_phi_using_exp}
\begin{aligned}
\frac{1}{\leb(B(z,R))}\int_{B(z,R)} \phi(\abs{z-u})\,du &= E \phi(\abs{z-U})\\
&=\int_{0}^{1}P\crbk{\phi(\abs{z-U})\ge t}dt.\\
&=\int_{0}^{1}\frac{\leb\big(\{u\in B(z,R):\phi(\abs{z-u})\ge t\}\big)}{\leb(B(z,R))} dt.
\end{aligned}
\end{equation}
where $U$ is a uniform r.v. on $B(z,R)$. Note that for $t\in[0,\phi(R))$,  $$\set{u\in\Z:\phi(|z-u|) \ge t}\supset B(z,R).$$ So, for all $z\in\Z$ 
\begin{equation}\label{eqn:tu_t_small_vol_bd}
\begin{aligned}
    \leb\big(\{u\in B(z,R):\phi(\abs{z-u})\ge t\}\big)& = \leb\big(B(z,R)\cap \{u\in\Z:\phi(\abs{z-u})\ge t\}\big)\\
    &=\leb(B(z,R))\\
    &\geq v_\wedge V_d\crbk{R}
\end{aligned}
\end{equation}
where the last equality follows from the local volume condition and $R\leq \sigma_\vee \sqrt{d}$. 

On the other hand, for $t\in[\phi(R),1]$, we have  $\sigma\sqrt{-2\log t} \leq R \leq \sigma_\wedge \sqrt{d}$ and $$\{u\in B(z,R):\phi(\abs{z-u})\ge t\}=\set{u\in\Z:\abs{z-u}\le \sigma\sqrt{-2\log t}}.$$

So the local volume condition yields
\begin{equation}\label{eqn:tu_t_large_vol_bd}
\leb\big(\{u\in B(z,R):\phi(\abs{z-u})\ge t\}\big)\ge v_\wedge V_d\crbk{\sigma\sqrt{-2\log t}}.
\end{equation}

Combining \eqref{eqn:tu_int_phi_using_exp}, \eqref{eqn:tu_t_small_vol_bd}, and \eqref{eqn:tu_t_large_vol_bd}, we have that for all $z\in\Z$
$$\begin{aligned}
\int_{B(z,R)} \phi(\abs{z-u})\,du 
&\ge v_\wedge\crbk{\phi(R)V_d\crbk{R} +  \int_{\phi(R)}^{1}V_d\crbk{\sigma\sqrt{-2\log t}} dt} \\
&=  v_\wedge\int_{0}^{\phi(R) }\leb\big(\set{u:|u-z|\leq R}\big)dt +  \int_{\phi(R)}^1\leb\big(\{u:\phi(\abs{z-u})\ge t\}\big) dt\\
&\stackrel{(i)}= v_\wedge\int_{0}^{\phi(R) }\leb\big(\set{u:|u|\leq R}\big)dt +  \int_{\phi(R)}^1\leb\big(\{u:\phi(\abs{u})\ge t\}\big) dt\\
&= v_\wedge \int_{|u|\leq R}\phi(|u|)du. 
\end{aligned}$$
where $(i)$ uses the invariance of Lebesgue measure under shifts, and the last inequality follows from a similar argument as in \eqref{eqn:tu_int_phi_using_exp}.

Let $G\sim N(0,I_d)$. Then
\begin{align*}
\int_{|u|\le R} \phi(|u|)du
&=(2\pi)^{d/2}\sigma^d\,P(\abs{G}\le R/\sigma)\\
&=(2\pi)^{d/2}\sigma^d\,P(\abs{G}^2\le d) \\
&\ge\tfrac12(2\pi)^{d/2}\sigma^d, 
\end{align*}
where the last inequality uses the fact that $|G|^2\sim\chi^2_d$ with median$(|G|^2)\leq d$.
Therefore, going back to \eqref{eqn:tu_c_lb_init}
\begin{equation}\label{eqn:denom_lower_final_case_i}
\inf_{z\in\Z}c_{\mu_b}(z)\ge \tfrac12\,\underline p_b\,v_\wedge\,(2\pi)^{d/2}\sigma^{d}.
\end{equation}

\textbf{Denominator lower bound, cases (ii).}
Fix $z=(x,a_i)\in\X\times\A$. Writing $\mu_b(dy, a_j)=p_b(y,a_j)\,dy$,
\[
c_{\mu_b}(x,a_i)=\sum_{j=1}^m\int_{\X}\exp\!\Big(-\frac{\abs{x-y}^2+\1_{\{i\neq j\}}}{2\sigma^2}\Big)\,p_b(y,a_j)\,dy
\;\ge\;\underline p_b\int_{\X}e^{-\abs{x-y}^2/(2\sigma^2)}\,dy.
\]
Applying the same argument as in case (i) on $\X\subset\R^{d_\X}$ (with the Euclidean local volume condition on $\X$)
yields
\begin{equation}\label{eqn:denom_lower_final_case_ii}
\inf_{(x,a)\in\Z}c_{\mu_b}(x,a)\ge \tfrac12\,\underline p_b\,v_\wedge\,(2\pi)^{d_\X/2}\sigma^{d_\X}.
\end{equation}

\textbf{Numerator upper bound, case (i).}
Here $\mathsf d(z,u)=\abs{z-u}$ on $\R^{d}$ with $d=d_\X+d_\A$. For any $z\in\Z$,
\[
\int_\Z \abs{z-u}^{\alpha}\kappa_\sigma(z,u)\mu_b(du)
\le \overline p_b \int_{\R^{d}}\abs{v}^{\alpha}\exp\!\Big(-\frac{\abs{v}^2}{2\sigma^2}\Big)\,dv.
\]
Let $G\sim N(0,I_d)$. Then the Gaussian moment gives
\begin{align*}
\frac{1}{(2\pi)^{d/2}\sigma^{d}}\int_{\R^{d}}\abs{v}^{\alpha}e^{-\abs{v}^2/(2\sigma^2)}dv
&=\sigma^\alpha\,\E[\abs{G}^{\alpha}]\\
&\le \sigma^\alpha\,(\E[\abs{G}^2])^{\alpha/2}\\
&=\sigma^\alpha\,d^{\alpha/2},   
\end{align*}
where we used Jensen's inequality. So, $$\int_{\R^{d}}\abs{v}^{\alpha}e^{-\abs{v}^2/(2\sigma^2)}dv\leq (2\pi)^{d/2}\sigma^{d} \sigma^\alpha\,d^{\alpha/2}$$

Combining this and \eqref{eqn:denom_lower_final_case_i} yields $$\xi_{\mu_b}(\sigma)\leq d^{\alpha/2} \frac{2\overline p_b}{v_\wedge\underline p_b} \sigma^\alpha.$$ This shows the case (i) in\eqref{eqn:xi_rate}.

\textbf{Numerator upper bound, case (ii).}
Let $\A=\{a_1,\dots,a_m\}$, $m:=|\A|$, and write $z=(x,a_i)$.  Set $D_\X:=\diam(\X):=\sup_{x,y\in\X}|x-y|<\infty$.
Using $p_b\le \overline p_b$ and splitting $j=i$ and $j\neq i$,
\begin{align*}
&\int_{\Z}\mathsf d(z,u)^{\alpha}\kappa_\sigma(z,u)\mu_b(du)\\
&\le \overline p_b\int_{\X}|x-y|^{\alpha}e^{-|x-y|^2/(2\sigma^2)}dy
\;+\;\overline p_b\sum_{j\neq i}\int_{\X}(|x-y|^2+1)^{\alpha/2}e^{-(|x-y|^2+1)/(2\sigma^2)}dy\\
&\le \overline p_b\int_{\R^{d}}|v|^{\alpha}e^{-|v|^2/(2\sigma^2)}dv
\;+\;\overline p_b(m-1)(D_\X^2+1)^{\alpha/2}e^{-1/(2\sigma^2)}\int_{\R^{d_\X}}e^{-|v|^2/(2\sigma^2)}dv\\
&= \overline p_b(2\pi)^{d_\X/2}\sigma^{d_\X}\sqbk{\sigma^{\alpha}\E|G|^{\alpha}+(m-1)(D_\X^2+1)^{\alpha/2}e^{-1/(2\sigma^2)}}\\
&\le \overline p_b(2\pi)^{d_\X/2}\sigma^{d_\X}\sqbk{\sigma^{\alpha}d_\X^{\alpha/2}+(m-1)(D_\X^2+1)^{\alpha/2}e^{-1/(2\sigma^2)}}
\end{align*}
where $G\sim N(0,I_{d_\X})$.

Dividing by the denominator lower bound
$\int_\Z\kappa_\sigma(z,u)\mu_b(du)\ge \underline p_b e^{-1/2}v_\wedge\sigma^{d}$ shown above yields, for all
$\sigma\in(0,\sigma_\vee]$,
\[
\xi_{\mu_b}(\sigma)\le \frac{2}{v_\wedge}\frac{\overline p_b}{\underline p_b}
\Big(d_\X^{\alpha/2} \sigma^{\alpha}+(m-1)(D_\X^2+1)^{\alpha/2}e^{-1/(2\sigma^2)}\Big).
\]
This and the inequality $a^2 + b^2\leq (|a|+|b|)^2$ implies case (ii) in \eqref{eqn:xi_rate}.
\end{proof}

\section{Numerical Experiment}\label{section:numerical_experiment}

We test Q-Measure-Learning (Algorithm~\ref{alg:Q_meas_learning}) on a two-item lost-sales inventory control problem with a continuous state space and a finite action set. To streamline the presentation, we defer some numerical and implementation details to Appendix~\ref{section:additional_numerial}.

\subsubsection*{The Inventory Model}
First, we specify the controlled Markov chain $\{(X_t,A_t):t\ge 0\}$ on $\Z=\X\times\A$. The state is the on-hand inventory vector $$X_t=(X_{t,1},X_{t,2})\in \X:=[0,I_{\max}]^2,$$ where $I_{\max}>0$ is the maximum inventory level (capacity) per item. The action is the order quantity vector $$A_t=(A_{t,1},A_{t,2})\in \A:=\{0,1,\dots,A_{1,\max}\}\times \{0,1,\dots,A_{2,\max}\},$$ so $\A$ is finite. The plots are constructed using an instance with $I_{\max}=15$ and $(A_{1,\max},A_{2,\max})=(10,8)$.

Let $\set{D_t=(D_{t,1},D_{t,2}):t\geq 0}$ be an i.i.d.\ sequence of nonnegative demand vectors. We induce correlation across coordinates by taking the absolute value of a correlated Gaussian vector, namely $D_t = |G_t|$ where $G_t\sim N(\mu,\Sigma)$. The lost-sales inventory dynamics then evolve as
\begin{equation}\label{eq:inventory_dynamics_capacity}
X_{t+1} = \min\set{(X_t+A_t - D_t)_+,I_{\max}},
\end{equation}
where $(\cdot)_+$ denotes the componentwise positive part.

Moreover, let $h,p,c\in\R^2_{>0}$ be holding, lost-sales, and unit ordering cost vectors, and let $k>0$ be a fixed ordering cost. The one-step cost is
\begin{equation}\label{eq:inventory_cost}
C_{t+1}
=
k\,\1\{A_{t,1}+A_{t,2}>0\}
+c^\top A_t+h^\top X_{t+1}
+p^\top(D_t-X_t-A_t)_+.
\end{equation}
Since $D$ is unbounded, to roughly match Assumption~\ref{assump:rwd}, we normalize costs to obtain bounded rewards; that is, $R_{t+1} = -C_{t+1}/C_{\max}$ where $C_{\max}$ satisfies $C_t\le C_{\max}$ with high probability.

In this setting, the dynamic programming principle holds with reward function $r(x,a) = E[R_{1}| X_0 = x, A_0 = a]$, and the optimal $Q^*$ is the unique solution to $\overline \cT [Q^*] = Q^*$, where $\overline \cT$ is given in \eqref{eqn:def_clipped_Bellman} with controlled transition kernel
$P(A|x,a) = E[\1\set{X_1\in A}|X_0 = x,A_0 = a]$ for all $A\in\cX$. 

\subsubsection*{Algorithm Setup and Benchmark}
We implement the weight-based Algorithm~\ref{alg:Q_meas_learning} using data generated from a single trajectory under a fully exploratory behavior policy. Specifically, we use the uniform behavior policy on the finite action set; that is, $\pi_b(a| x)=1/|\A|$ for all $a\in\A$ and $x\in\X$. Under this exploration policy, we empirically verify that the state space is well explored, which supports satisfaction of the density assumptions in Theorem~\ref{thm:approx_Q_err} (see Figure~\ref{fig:visitation} in Appendix~\ref{section:additional_numerial}).

Even though $\A$ is finite, we view it as a compact subset of $\R^2$ and implement the case (i) kernel in Section~\ref{section:gaussian_kernel}. For the plots shown in this paper, we set $\sigma=1$. We remark that during additional experimentation, we obtain similar convergence and policy error behavior for $\sigma\in[0.5,2]$. In contrast, choosing $\sigma$ too small (e.g.\ $\sigma=0.1$) under-smooths $q_n$, yielding a substantially rougher estimate and slower convergence. Moreover, we choose the stepsizes following Assumption~\ref{assump:stepsize}, with $\beta_n=1/(n+1)$ and $\alpha_n\sim 1/n$.

Since we have correlated demand and a positive fixed ordering cost, we are not aware of a closed-form solution. Instead, we compute an approximation $Q_\mrm{DP}$ to the optimal $Q^*$ using dynamic programming on a quantized state space, with a Monte Carlo estimate of the demand distribution.

\subsubsection*{Convergence and Policy Behavior}
\begin{figure}[tb]
\centering
\includegraphics[width=0.98\textwidth]{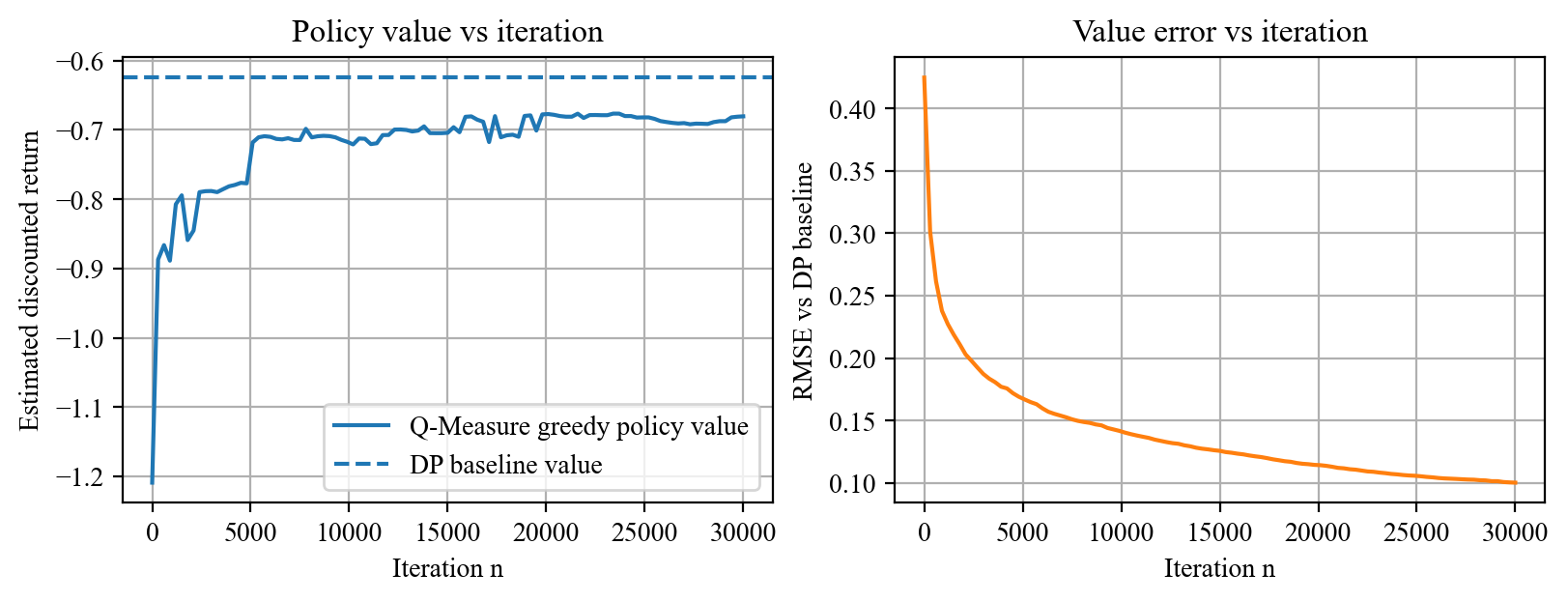}
\caption{Left: estimated discounted return of the greedy policy induced by $q_n$. Right: $Q^*$ estimation error, where the curves exhibit stabilization/decrease consistent with convergence.}
\label{fig:inventory_value_rmse}
\end{figure}

Figure~\ref{fig:inventory_value_rmse} reports two convergence diagnostics computed during training: the estimated discounted return of the greedy policy induced by $q_n$, and a root mean squared error (RMSE) curve that tracks the error $q_n-Q_{\mathrm{DP}}$ on a fixed evaluation grid of state-action pairs. We note that both plots are generated from a \textit{single training trajectory}. The plots show that the return increases and the RMSE decreases with the number of completed iterations $n$. This is consistent with the convergence mechanism in Section~\ref{section:convergence}: as the coupled trackers $(\mu_n,\nu_n)$ stabilize, the induced function $q_n=\Phi_{\mu_n}[\nu_n]$ approaches its limiting fixed point $q^*$, and the greedy policy correspondingly stabilizes.

We also note that both plots suggest a persistent gap to the optimal return and to $0$ RMSE. This is consistent with our theory: the smoothing parameter $\sigma>0$ induces a strictly positive approximation error, so that $\norm{q^*-Q^*}>0$.

\begin{figure}[htb]
\centering
\includegraphics[width=0.98\textwidth]{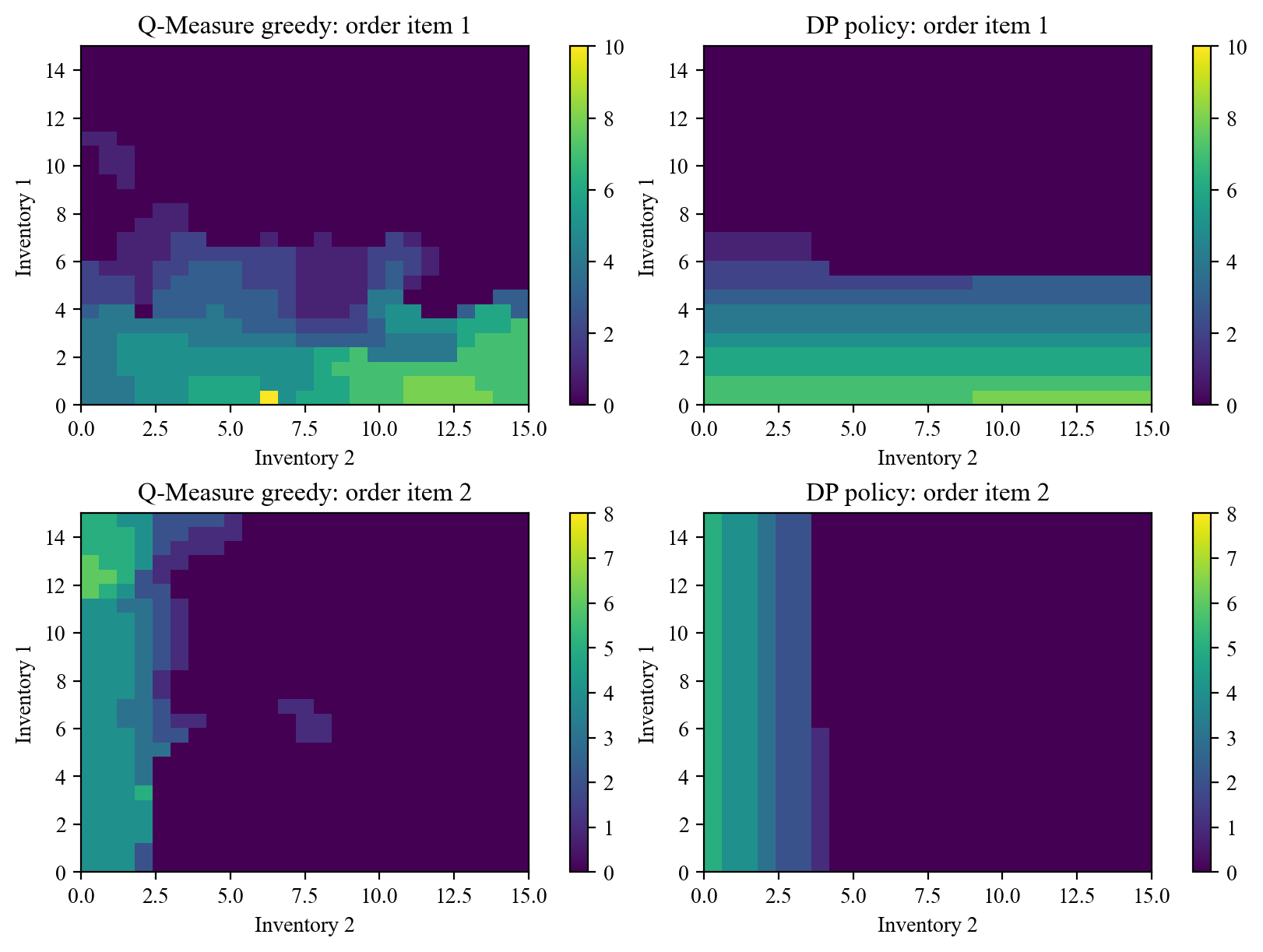}
\caption{Greedy policy induced by $q_n$ compared with the DP policy on a grid.}
\label{fig:inventory_policy_compare}
\end{figure}

Figure~\ref{fig:inventory_policy_compare} visualizes the learned greedy policy over the inventory state space and compares it to the same offline benchmark policy. Both policies exhibit a similar structure, as is typical for lost-sales systems: when inventory is low, the policy places positive orders; when inventory is sufficiently high, the policy tends to place no order. The close agreement in the induced partition of $\X$ into ordering and non-ordering regions supports the qualitative correctness of the learned policy.

\FloatBarrier

\bibliographystyle{apalike}
\bibliography{ref}

\newpage
\appendix
\appendixpage

\section{Proof of Proposition \ref{prop:D_metric}}\label{section:proof:prop:D_metric}

We will use the following auxiliary results from functional analysis. 
\begin{lemma}[\citet{rudin1974R_and_C}, Theorem 6.19]\label{lemma:riesz}
Let $\Z$ be compact Hausdorff and let $\mu\in\cM(\cZ)$. If $\int_\Z fd\mu=0$ for all $f\in C(\Z)$, then $\mu=0$.
\end{lemma}

Note that this is the uniqueness statement in the Riesz representation theorem for $C(\Z)$. 

\begin{lemma}[\citet{steinwart2001influence}, Example 1]\label{lemma:steinwart}
Fix $\sigma>0$. For every nonempty compact $K\subset\R^d$, the Gaussian kernel
$\kappa(x,y)=\exp(-\abs{x-y}^2/(2\sigma^2))$ is universal on $K$; i.e. for every $f\in C(K)$ and $\epsilon> 0$, there exists a function $g$ induced by $\kappa$ with $\norm{f-g}\leq\epsilon$.
\end{lemma}
Equipped with these results, we prove Proposition \ref{prop:D_metric}.
\begin{proof}[Proof of Proposition \ref{prop:D_metric}]
 First, we note that it is straightforward to show that $D_\mu$ is non-negative, symmetric, and satisfies the triangle inequality. We check that $D_{\mu}(\nu,\nu')=0$ implies $\nu = \nu'$.
Assume $D_{\mu}(\nu,\nu')=0$ and set $m:=\nu-\nu'$.

Next, we check the universality of $\kappa$ on $\Z$.
In case (i), $\Z\subset\R^{d_\X+d_\Z}$ is compact, so Lemma~\ref{lemma:steinwart} applies with $d=d_\X+d_\Z$. In case (ii), define the invertible map $\psi:\X\times\A\to\R^{d_\X+m}$ by $\psi(x,a_i)=(x,e_i/\sqrt{2})$, where $(e_i)$ is the standard basis of $\R^m$.
Then $\psi(\cZ)$ is compact and for all $z,u\in\cZ$ and
\[
\kappa(z,u)=\exp\crbk{-\frac{\abs{\psi(z)-\psi(u)}^2}{2\sigma^2}}
\]
Thus $\kappa_\sigma$ is the Gaussian kernel on the compact set $\psi(\Z)\subset\R^{d+m}$,
which is universal on $\psi(\Z)$.

Define $L:C(\Z)\to\R$ by $L[f]:=\int_{\Z} fdm$; then $\abs{L[f]}\le \norminf{f}\normTV{m}$, so $L$ is continuous.

To argue $m = 0$, we prove by contradiction and assume that $\TV{m}>0$. Consider arbitrary $f\in C(\Z)$ and $\eta>0$. We show that $L[f] <\eta$. By universality, there exist a Hilbert space $(H,\angbk{\cd,\cd}_H)$, a feature map $\Psi:\Z\to H$
with $\kappa(x,y)=\angbk{\Psi(x),\Psi(y)}_H$, and an induced function $g(\cdot)=\angbk{w,\Psi(\cdot)}_H$ such that
\[
\norm{f-g}<\frac{\eta}{2\TV{m}}.
\]
Let $H_0:=\overline{\spn\{\Psi(z):z\in\Z\}}$ and write $H\ni w=w_0+w_\perp$ with $w_0\in H_0$ and $w_\perp\in H_0^\perp$.
Then $g(\cdot)=\angbk{w_0,\Psi(\cdot)}_H$.
Since $w_0\in H_0$, we can choose $$w_n=\sum_{i=1}^n \alpha_i\Psi(z_i)\in \spn\{\Psi(z):z\in\Z\}$$ such that
$\norm{w_0-w_n}_H<\eta/(2M\normTV{m})$, where $M:=\sup_{z\in\Z}\norm{\Psi(z)}_H=\sup_{z}\sqrt{\kappa(z,z)}= 1$.
Then the function $h(\cdot):=\angbk{w_n,\Psi(\cdot)}_H=\sum_{i=1}^n \alpha_i\kappa(z_i,\cdot)$ satisfies
$$\norm{g-h}\le M\norm{w_0-w_n}_H<\eta/(2\normTV{m}).$$

Since $\kappa$ is positive on compact $\Z$, $D_{\mu}(\nu,\nu')=0$ implies that
$\int_{\Z}\kappa(z,u)m(du)=0,$ for all $z\in\Z$. So, $L[\kappa(z_i,\cdot)]=0$ for all $i$, and we have $L[h]=0$. Therefore
\[
\abs{L[f]}
\le \abs{L[f-g]}+\abs{L[g-h]}+\abs{L[h]}
\le \normTV{m}\crbk{\norm{f-g}+\norm{g-h}}
<\eta.
\]
As $\eta>0$ is arbitrary, $L[f]=0$ for all $f\in C(\Z)$. But by Lemma~\ref{lemma:riesz}, this can only happen if $m=0$, contradicting our assumption that $\TV{m} >0$. Hence we must have that $m = 0$; i.e. $\nu=\nu'$.
\end{proof}

\section{Proof of Proposition \ref{prop:vanishing_err}}\label{section:proof:prop:vanishing_err}
We first introduce two classical results for martingales and Markov chains.
\begin{lemma}[\citet{Doob1953}, Chapter VII, Theorem 4.1]\label{lemma:Doob_L2_conv}
Let $(M_n,\cF_n)_{n\ge 0}$ be a real-valued martingale such that
$\sup_{n\ge 0}E[M_n^2]<\infty$. 
Then $M_n$ converges a.s. and in $L^2$ as $n\ra\infty$ to a square integrable random variable $M_\infty$. In particular, converging real sequences are Cauchy, so $\sup_{m\ge n}\abs{M_m-M_n}\ra 0$ a.s. as $n\ra\infty$. 
\end{lemma}

\begin{lemma}[Poisson's equation]\label{lemma:poisson_eqn}
Assume Assumption~\ref{assump:unif_ergodicity}. Let $f\in C(\Z)$ and set $ f_c:=f-\mu_b[f]$.
Define 
\begin{equation}\label{eqn:Poisson_solution_def_short}
v_f(z):=\sum_{t=0}^{\infty} P_b^t [f_c](z),\qquad z\in\Z.
\end{equation}
Then, the series \eqref{eqn:Poisson_solution_def_short} converges absolutely and uniformly on $\Z$,
$v_f$ is bounded measureable, and it solves the Poisson equation $v_f - P_b [v_f] = f_c$
Moreover,
\begin{equation}\label{eqn:Poisson_sup_bound_short}
\norm{v_f}\le \frac{2c}{1-\rho}\norm{f}.
\end{equation}
\end{lemma}

\begin{proof}[Proof of Lemma \ref{lemma:poisson_eqn}]
By Assumption~\ref{assump:unif_ergodicity}, for any bounded $h$ and any $z\in\Z$,
\[
\abs{P_b^t [h](z)-\mu_b[h]}
\le \norm{h}\normTV{P_b^t(z,\cd)-\mu_b(\cd)}
\le c\rho^t \norm{h}.
\]
Apply this with $h=f_c$ and use $\mu_b[f_c]=0$ to get
$\norm{P_b^t[f_c]}\le c\rho^t\norm{f_c}\le 2c\rho^t\norm{f}$.
Hence $\sum_{t\ge 0}\norm{P_b^tf_c}<\infty$, which implies absolute and uniform convergence of
\eqref{eqn:Poisson_solution_def_short} and yields \eqref{eqn:Poisson_sup_bound_short}.
Then, absolute and uniform convergence justify the termwise application of $P_b$, giving
$P_b [v_f]=\sum_{t\ge 0}P_b^{t+1}[f_c]=\sum_{t\ge 1}P_b^{t}[f_c]$ and thus
$v_f-P_b[v_f]=f_c$.
\end{proof}

\begin{proof}[Proof of Proposition~\ref{prop:vanishing_err}]
Define for $m\ge n$,
\[
S^{D}_{n,m}:=\sum_{k=n}^{m-1}\alpha_{k+1}D_{k+1},\qquad
S^{G}_{n,m}:=\sum_{k=n}^{m-1}\alpha_{k+1}G_{k},\qquad
S^{B}_{n,m}:=\sum_{k=n}^{m-1}\alpha_{k+1}B_{k}.
\]
We will show the a.s. convergence for each of these sums. 
\paragraph{The bias term $S^B_{n,m}$:}
By definition of $\cN(n,T)$, for every $m\in\cN(n,T)$ one has
$\sum_{k=n}^{m-1}\alpha_{k+1}\le T$, hence
\begin{equation}\label{eqn:B_window_bound}
\max_{m\in\cN(n,T)}\norm{S^{B}_{n,m}}
\le
\sum_{k=n}^{\max\cN(n,T)-1}\alpha_{k+1}\norm{B_k}
\le T\sup_{k\ge n}\norm{B_k}.
\end{equation}
For each $k$, using the definition of $B_k$ in \eqref{eqn:def_gn_bn}, the bound
$\norm{\Phi_{\mu_b}[\delta_{Z_k}]}\le \kappa_\wedge^{-1}$,
and the $\gamma$-contraction of $\overline\cT$,
\[
\norm{B_k}
\le \frac{1}{\kappa_\wedge}\norm{\overline\cT[q_k]-\overline\cT[\bar q_k]}
\le \frac{\gamma}{\kappa_\wedge}\norm{q_k-\bar q_k}.
\]
Lemma~\ref{lemma:q_bar_approx_q} gives $\norm{q_k-\bar q_k}\to 0$ a.s., hence
$\sup_{k\ge n}\norm{B_k}\to 0$ a.s. Therefore, from \eqref{eqn:B_window_bound}, we have
\begin{equation}\label{eqn:B_window_vanish}
\max_{m\in\cN(n,T)}\norm{S^{B}_{n,m}}\ra 0
\end{equation}
a.s. as $n\ra\infty$. 

\paragraph{A $\delta$-net argument:} 
Our strategy to control the sup error of $S^D_{m,n}$ and $S^G_{m,n}$ starts from observing that these errors are uniformly Lipschitz in $z$. Therefore, for any $\delta > 0$ we can construct a finite $\delta$-net so that we only need to track the errors at those points within the net.  

Specifically, since $\Z$ is compact, for all $z,z',u\in\Z$ the mean value theorem gives
\[
\abs{\kappa(z,u)-\kappa(z',u)}
\le \frac{\diam(\Z)}{\sigma^2}\abs{z-z'}.
\]
Here we abuse the notation when $\A$ is finite to denote $|(x,a_i)-(x',a_j)| = |x-x'| + \1_{i\neq j}$. 

Set $L_\kappa:=\diam(\Z)/\sigma^2$. Recall that in \eqref{eqn:def_c_mu}, we define $c_{\mu_b}(z) = \int_\Z \kappa(z,u)\mu_b(du)$. So, $z\ra c_{\mu_b}(z)$ is $L_\kappa$-Lipschitz.
Therefore, for every $u,z,z'\in\Z$, 
$$\begin{aligned}\abs{\Phi_{\mu_b}[\delta_u](z) -\Phi_{\mu_b}[\delta_u](z')} &=\abs{\frac{\kappa(z,u)}{c_{\mu_b}(z)} - \frac{\kappa(z',u)}{c_{\mu_b}(z')}}\\
    &\leq \abs{\frac{\kappa(z,u) - \kappa(z',u)}{c_{\mu_b}(z)}}+\abs{ \frac{\kappa(z',u)(c_{\mu_b}(z)-c_{\mu_b}(z'))}{c_{\mu_b}(z)c_{\mu_b}(z')}}\\
    &\leq L_\kappa \crbk{\frac{1}{\kappa_\wedge} + \frac{1}{\kappa_\wedge^2}}|z-z'|;
\end{aligned}$$
i.e. $\Phi_{\mu_b}[\delta_u]$ is Lipschitz with constant $L_\Phi:=L_\kappa(\frac{1}{\kappa_\wedge}+\frac{1}{\kappa_\wedge^2})$. Moreover, $\norm{\Phi_{\mu_b}[\delta_u]}\le \frac{1}{\kappa_\wedge}.$

\begin{comment}Thus, for any bounded measurable $f$,
\begin{equation}\label{eqn:K_Lipschitz_short}
\abs{\cK_{\mu_b}[f](z)-\cK_{\mu_b}[f](z')}
=\abs{\int_\Z f(u)\big(\Phi_u(z)-\Phi_u(z')\big)\mu_b(du)}
\le \norm{f}L_\Phi\norm{z-z'}.
\end{equation}
\end{comment}

Also, by Lemma~\ref{lemma:bounded_trackers}, $\abs{Y_{k+1}}\le \frac{1}{1-\gamma}$ and hence $\abs{E[Y_{k+1}| \cF_k]}\le \frac{1}{1-\gamma}$. Thus $D_{k+1}=\big(Y_{k+1}-E[Y_{k+1}| \cF_k]\big)\Phi_{\mu_b}[\delta_{Z_k}]$ satisfies
\begin{equation}\label{eqn:D_uniform_and_Lip_short}
\norm{D_{k+1}}\le \frac{2}{(1-\gamma)\kappa_\wedge},
\qquad
\abs{D_{k+1}(z)-D_{k+1}(z')}\le \frac{2L_\Phi}{1-\gamma}\abs{z-z'}.
\end{equation}
Moreover, with $\norm{\overline\cT_{\mu_b}[\bar q_k]}\le \frac{1}{1-\gamma}$ and $\norm{\overline\cT[\bar q_k]}\le \frac{1}{1-\gamma}$, 
the Markov term
$G_k=\overline\cT[\bar q_k](Z_k)\Phi_{\mu_b}[\delta_{Z_{k}}]-\overline\cT_{\mu_b}[\bar q_k]$
satisfies the uniform Lipschitz bound
\begin{equation}\label{eqn:G_Lip_short}
\begin{aligned}
\abs{G_k(z)-G_k(z')}
&\le \frac{L_\Phi}{1-\gamma}\abs{z-z'} + \cK_{\mu_b}[\overline\cT[\bar q_k]](z) -\cK_{\mu_b}[\overline\cT[\bar q_k]](z')\\
&\leq \frac{L_\Phi}{1-\gamma}\abs{z-z'} + \frac{\int_\Z \abs{c_{\mu_b}(z')\kappa(z,u) -c_{\mu_b}(z)\kappa(z',u)} \overline\cT[\bar q_k](u)\mu_b(du)}{c_{\mu_b}(z)c_{\mu_b}(z')}\\
&\leq \frac{2L_\Phi}{1-\gamma}|z-z'|.
\end{aligned}
\end{equation}

Fix $\epsilon>0$.
Let $L_0:=\frac{2}{1-\gamma}L_\Phi,$ and $
\delta:=\frac{\epsilon}{L_0 T}.$ Since $\Z$ is compact, there exists a finite $\delta$-net $\{z_1,\dots,z_l\}\subset\Z$ such that for any $z\in\Z$, there exists $j\leq N$ with $|z - z_j|\leq\delta$. Then, for any function $h:\Z\to\R$ that is $L$-Lipschitz,
\begin{equation}\label{eqn:net_reduction}
\norm{h}
=\sup_{z\in\Z}\abs{h(z)}
\le \max_{1\le j\le l}\abs{h(z^{(j)})}+L\delta.
\end{equation}

If $m\in\cN(n,T)$, then $\sum_{k=n}^{m-1}\alpha_{k+1}\le T$, hence by \eqref{eqn:D_uniform_and_Lip_short}--\eqref{eqn:G_Lip_short}
the functions $S^{D}_{n,m}$ and $S^{G}_{n,m}$ are both $L_0T$-Lipschitz. Therefore, for all $n$,
\begin{equation}\label{eqn:net_reduction_window}
\begin{aligned}
\max_{m\in\cN(n,T)}\norm{S^{D}_{n,m}}
&\le \max_{1\le j\le l}\max_{m\in\cN(n,T)}\abs{S^{D}_{n,m}(z_j)}+\epsilon,\\
\max_{m\in\cN(n,T)}\norm{S^{G}_{n,m}}
&\le \max_{1\le j\le l}\max_{m\in\cN(n,T)}\abs{S^{G}_{n,m}(z_j)}+\epsilon.
\end{aligned}
\end{equation}
Thus, it suffices to control the error processes at the finitely many points $z_1,\ds, z_l$.

\paragraph{The martingale term $D_{k+1}$.}
Fix $j\in\{1,\dots,l\}$ and consider the scalar process
$\set{S^{D}_{0,m}(z_j):m\geq 0}$. 
Since $E[D_{k+1}(z_j)|\cF_k]=0$ by definition of $D_{k+1}$, $\set{S^{D}_{0,m}(z_j):m\geq 0}$ is a $\cF_m$-martingale.
Moreover, using \eqref{eqn:D_uniform_and_Lip_short} and orthogonality of martingale differences,
\begin{align*}
E\big[(S^{D}_{0,m}(z_j))^2\big]
&=\sum_{k=0}^{m-1}\alpha_{k+1}^2 E\big[D_{k+1}(z_j)^2\big]\\
&\le \frac{4}{(1-\gamma)^2\kappa_\wedge^2}\sum_{k=0}^{\infty}\alpha_{k+1}^2<\infty.
\end{align*}
Hence, Lemma~\ref{lemma:Doob_L2_conv} holds and we conclude that
\[
\max_{m\in\cN(n,T)}\abs{S^{D}_{n,m}(z_j)}
\le \sup_{m\ge n}\abs{S^{D}_{0,m}(z_j) - S^{D}_{0,n}(z_j)}\ra 0
\]
a.s. as $n\ra\infty$. 

Taking the maximum over the finite set $j=1,\dots,l$ preserves almost sure convergence. Hence
\eqref{eqn:net_reduction_window} yields that for any $\epsilon > 0$, there exists $\Omega_\epsilon$ for s.t. $\Omega^c_\epsilon$ is $P$-null and for all $\omega\in\Omega_\epsilon$
\[
\limsup_{n\ra\infty}\max_{m\in\cN(n,T)}\norm{S^{D}_{n,m}}\leq\epsilon.
\]
Let $\Omega_* = \bigcap_{k=1}^\infty\Omega_{1/k}$. Clearly, $P(\Omega_*=1)$ and for all $\omega\in\Omega_*$
\begin{equation}\label{eqn:D_window_vanish}
\limsup_{n\ra\infty}\max_{m\in\cN(n,T)}\norm{S^{D}_{n,m}}=0
\end{equation}
showing a.s. convergence. 

\paragraph{The Markov term $G_k$ via Poisson's equation.} Define $\forall u\in\Z$, 
\[
f_{j,k}(u):=\overline\cT[\bar q_k](u)\Phi_{\mu_b}[\delta_u](z_j).
\]
Notice that $$\mu_b[f_{j,k}] = \int_{\Z}\overline\cT[\bar q_k](u)\frac{\kappa (u,z_j)}{c_{\mu_b}(z_j)}\mu_b(du) =\overline{\cT}_{\mu_b}[\bar q_k](z_j). $$
Then, by the definition of $G_k$,
\begin{equation}\label{eqn:G_as_centered_f_simplified}
G_k(z_j)=f_{j,k}(Z_k)-\mu_b[f_{j,k}].
\end{equation}
Moreover, using $\norm{\overline\cT[\bar q_k]}\le \frac{1}{1-\gamma}$ and
$\norm{\Phi_{\mu_b}[\delta_u]}\le \kappa_\wedge^{-1}$, we have
\begin{equation}\label{eqn:f_uniform_bound_simplified}
\norm{f_{j,k}}\le \frac{1}{(1-\gamma)\kappa_\wedge}.
\end{equation}

Let $v_{j,k}:\Z\to\R$ be the solution to Poisson's equation associated with $f_{j,k}$. By Lemma \ref{lemma:poisson_eqn},
\begin{equation}\label{eqn:vk_def}
v_{j,k}=\sum_{t=0}^{\infty}P_b^t[f_{j,k}-\mu_b[f_{j,k}]]
\end{equation}
where the sum converges absolutely. Since $f_{j,k}$ is $\cF_k$-measurable, so is $v_{j,k}$ is also $\cF_k$-measurable.

Also by Lemma~\ref{lemma:poisson_eqn},
$$v_{j,k}-P_bv_{j,k}=f_{j,k}-\mu_b[f_{j,k}],$$ and 
\begin{equation}\label{eqn:vk_uniform_bound}
\begin{aligned}
\norm{v_{j,k}}&\le \frac{2c}{1-\rho}\norm{f_{j,k}}\\
&\le \frac{2c}{1-\rho}\cdot\frac{1}{(1-\gamma)\kappa_\wedge}\\
&=:b_v.
\end{aligned}
\end{equation}
Therefore, by \eqref{eqn:G_as_centered_f_simplified} and the Markov property of $Z_k$, 
\begin{align*}
G_k(z_j)&=v_{j,k}(Z_k)-P_b[v_{j,k}](Z_k) \\
&= v_{j,k}(Z_k)-E\!\big[v_{j,k}(Z_{k+1})| \cF_k\big].   
\end{align*}
Define the martingale difference
$\eta_{j,k+1}:=v_{j,k}(Z_{k+1})-E\sqbkcond{v_{j,k}(Z_{k+1})}{ \cF_k}.$ Then $E[\eta_{j,k+1}| \cF_k]=0$, $\abs{\eta_{j,k+1}}\le 2\norm{v_{j,k}}\le 2b_v$, and
\begin{equation}\label{eqn:G_decomposition_simplified}
G_k(z_j)=v_{j,k}(Z_k)-v_{j,k}(Z_{k+1})+\eta_{j,k+1}.
\end{equation}
Consequently, for $m\ge n$,
\begin{equation}\label{eqn:G_sum_split_simplified}
\begin{aligned}
S^{G}_{n,m}(z_j)
&=\sum_{k=n}^{m-1}\alpha_{k+1}G_k(z_j)\\
&=
\underbrace{\sum_{k=n}^{m-1}\alpha_{k+1}\eta_{j,k+1}}_{=:M_{j,n,m}}
+
\underbrace{\sum_{k=n}^{m-1}\alpha_{k+1}\big(v_{j,k}(Z_k)-v_{j,k}(Z_{k+1})\big)}_{=:T_{j,n,m}}
\end{aligned}
\end{equation}
where $M_{j,n,m}$ is the martingale part and $T_{j,n,m}$ is the telescoping sum part.

\textit{The martingale part:} Note that $M_{j,0,m}$ is a $\cF_m$-martingale and
\[
E[M_{j,0,m}^2]=\sum_{k=0}^{m-1}\alpha_{k+1}^2E[\eta_{j,k+1}^2]
\le 4b_v^2\sum_{k=0}^{\infty}\alpha_{k+1}^2<\infty.
\]
Hence by Lemma~\ref{lemma:Doob_L2_conv}
\[
\max_{m\in\cN(n,T)}\abs{M_{j, n,m}}
\le \sup_{m\ge n}\abs{M_{j,0,m}-M_{j,0,n}}\ra0\quad\text{a.s.}
\]

\emph{The telescoping part:} For all $m\ge n$,
\begin{equation}\label{eqn:T_expand_simplified}
\begin{aligned}
T_{j,n,m}
&=\alpha_{n+1}v_{j,n}(Z_n)-\alpha_m v_{j,m-1}(Z_m)
+\sum_{k=n+1}^{m-1}(\alpha_{k+1}-\alpha_k)v_{j,k}(Z_k)\\
&\quad+\sum_{k=n+1}^{m-1}\alpha_k(v_{j,k}(Z_k)-v_{j,k-1}(Z_k)).
\end{aligned}
\end{equation}
Since $\abs{v_{j,k}(Z_k)}\le \norm{v_{j,k}}\le b_v$ and $\alpha_k$ is non-increasing, the sum of the first three terms in \eqref{eqn:T_expand_simplified} is bounded by $3b_v\alpha_{n+1}$.

For the last sum, we use the representation of $v_{j,k}$ in Lemma~\ref{lemma:poisson_eqn} and \eqref{eqn:vk_def}
\begin{align*}
\norm{v_{j,k}-v_{j,k-1}}&=\norm{\sum_{t=0}^{\infty}P_b^t[f_{j,k}- f_{j,k-1}-\mu_b[f_{j,k}-f_{j,{k-1}}]]}\\
&\le \frac{4c}{1-\rho}\norm{f_{j,k}-f_{j,k-1}}.
\end{align*}
Moreover,
\begin{align*}
\norm{f_{j,k}-f_{j,k-1}}
&\le \sup_{u\in\Z}\abs{\Phi_{\mu_b}[\delta_u](z_j)}
\norm{\overline\cT[\bar q_k]-\overline\cT[\bar q_{k-1}]}\\
&\le \frac{\gamma}{\kappa_\wedge}\norm{\bar q_k-\bar q_{k-1}}.   
\end{align*}
From \eqref{eqn:tu_Delta_initial_split},
$\bar q_k-\bar q_{k-1}=\alpha_k\big(Y_k\Phi_{\mu_b}[\delta_{Z_{k-1}}]-\bar q_{k-1}\big)$. Hence, by Lemma \ref{lemma:bounded_trackers},
\[
\norm{\bar q_k-\bar q_{k-1}}
\le \alpha_k\crbk{\frac{1}{(1-\gamma)\kappa_\wedge}+\frac{1}{(1-\gamma)\kappa_\wedge}}
=\frac{2\alpha_k}{(1-\gamma)\kappa_\wedge}.
\]
Combining the last three displays yields
\begin{equation}\label{eqn:vk_diff_alpha_simplified}
\norm{v_{j,k}-v_{j,k-1}}
\le
\frac{8c\gamma}{(1-\rho)(1-\gamma)\kappa_\wedge^2}\,\alpha_k.
\end{equation}

Therefore, for all $m\ge n$,
\[
\abs{T_{j,n,m}}
\le 3b_v\alpha_{n+1}
+\frac{8c\gamma}{(1-\rho)(1-\gamma)\kappa_\wedge^2}\sum_{k=n+1}^{\infty}\alpha_k^2
\ra0.
\]
as $n\ra\infty$. In particular, $\max_{m\in\cN(n,T)}\abs{T_{j,n,m}}\to 0$ everywhere.  Combining this with the martingale part and recalling \eqref{eqn:G_sum_split_simplified} yields $\max_{m\in\cN(n,T)}\abs{S^G_{n,m}(z_j)}\to 0$ a.s.

Applying the same strategy as in \eqref{eqn:D_window_vanish}, we convert this a.s. convergence $\max_{m\in\cN(n,T)}\abs{S^G_{n,m}(z_j)}\to 0$ to 
\begin{equation}\label{eqn:G_window_vanish}
\limsup_{n\ra\infty}\max_{m\in\cN(n,T)}\norm{S^{G}_{n,m}}=0 \quad \text{w.p.1. }
\end{equation}

\smallskip
\noindent\textbf{Step 5: conclude the proposition.}
Combining \eqref{eqn:B_window_vanish}, \eqref{eqn:D_window_vanish}, and \eqref{eqn:G_window_vanish} we obtain
\begin{align*}
&\max_{m\in\cN(n,T)}\norm{\sum_{k=n}^{m-1}\alpha_{k+1}(D_{k+1}+G_k+B_k)}\\
&\quad \leq \max_{m\in\cN(n,T)}\norm{S^{D}_{n,m}}+\max_{m\in\cN(n,T)}\norm{S^{G}_{n,m}}+\max_{m\in\cN(n,T)}\norm{S^{B}_{n,m}} \ra 0
\end{align*}
a.s. as $n\ra\infty$, proving \eqref{eqn:uniform_noise_control}.
\end{proof}

\section{Numerical Experiment Details and Additional Testing}\label{section:additional_numerial}

\subsection{Parameter Specification and Visualization Details}

We choose the discount factor $\gamma=0.7$, the bandwidth $\sigma=1$, and the initial state $X_0=(0,0)$.

Recall that $c,h,p,k$ are the ordering, holding, lost-sale, and fixed ordering costs, respectively. For the experiment presented in Section~\ref{section:numerical_experiment}, we use
$$
c=(0.3,0.35),\quad h=(0.05,0.04),\quad p=(0.8,0.9),\quad \text{and}\quad k=0.2.
$$

The demand is given by $D=|G|$ with $G\sim N(\mu,\Sigma)$, where
$$
\mu=(5,4),\quad \text{and}\quad \Sigma=\bmx{3&-0.21\\ -0.21 & 1}.
$$
It is not hard to see that the demand in this setup is negatively correlated.

The benchmark approximation of $Q^*$, denoted by $Q_{\mrm{DP}}$, is computed via dynamic programming on a discretized state space. We uniformly partition the interval $[0,I_{\max}=15]$ into $25$ sub-intervals, resulting in $625$ squares, each of which is represented as a state. The actions are $\A=\set{0,1,\ds,A_{1,\max}=10}\times\set{0,1,\ds,A_{2,\max}=8}$, so that $|\A|=99$. The DP procedure uses value iteration to compute $Q_{\mrm{DP}}$ on these $625\times 99=61875$ points, and then holds the value constant within each square corresponding to a state.

For Figure~\ref{fig:inventory_value_rmse}, the policy value is computed via Monte Carlo by averaging $256$ value estimates, each constructed by simulating the Markov chain induced by the current policy
\begin{equation}\label{eqn:greedy_policy}
    \pi_n(x) := \argmax{(a_1,a_2)\in\A} q_n(x,(a_1,a_2)) = \argmax{(a_1,a_2)\in\A} \Phi_{\mu_n}[\nu_n](x,(a_1,a_2))
\end{equation}
for $T=200$ steps; that is,
$$
\bar V=\frac{1}{256}\sum_{j=1}^{256} V_j;\qquad 
V_j:=\sum_{t=0}^T \gamma^t R_{j,t+1}.
$$

For Figure~\ref{fig:inventory_policy_compare}, we compute and plot the last-iterate policy (at $n=30000$) as in \eqref{eqn:greedy_policy}, and compare it against the greedy policy induced by $Q_{\mrm{DP}}$. The plots separately visualize the ordering quantities $a_1$ and $a_2$ for both policies on the $(x_1,x_2)$ plane.

\subsection{Additional Diagnostic Figures for Section \ref{section:numerical_experiment}}

\begin{figure}[ht]

\centering
\includegraphics[width=0.6\textwidth]{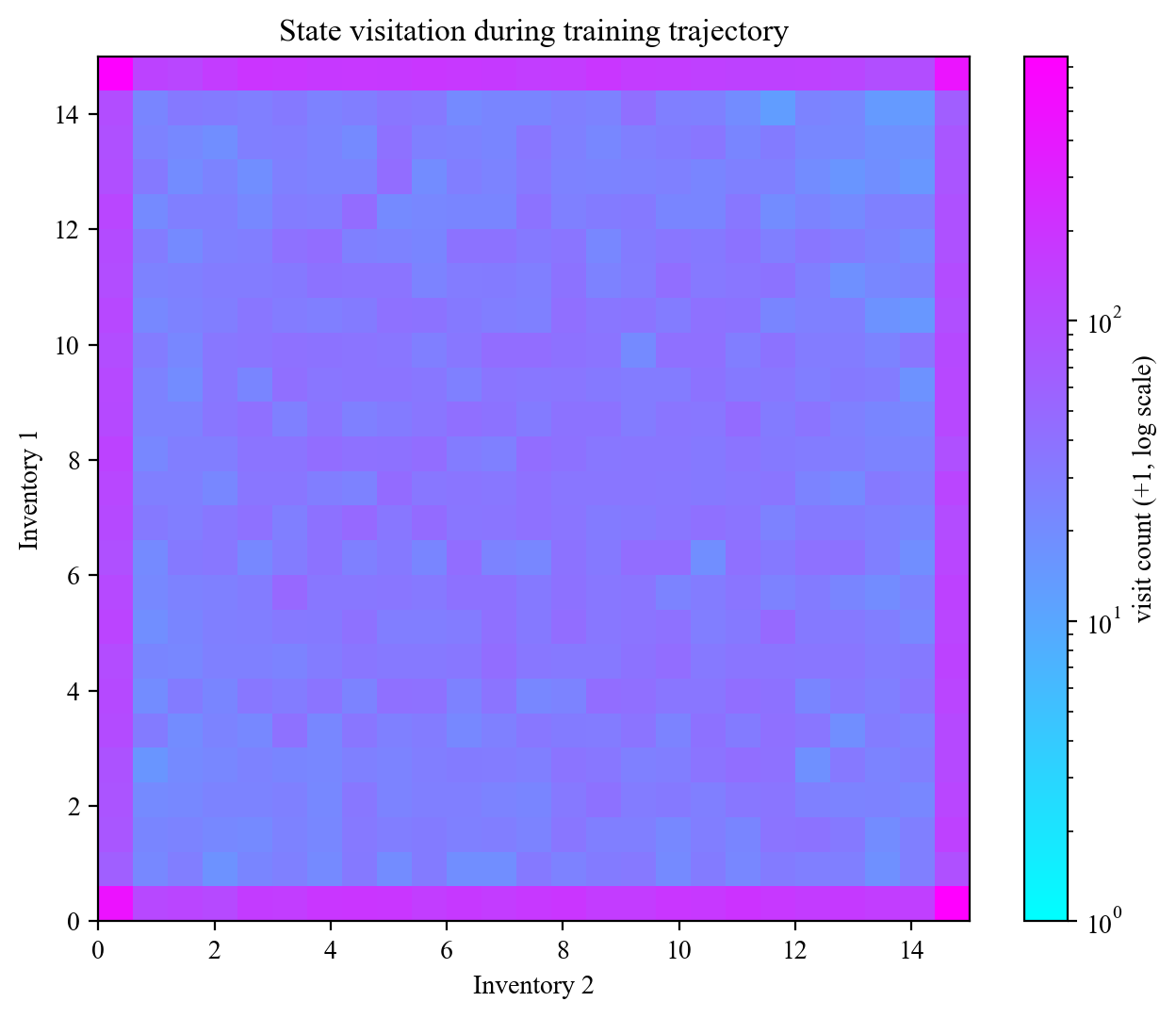}
\caption{Visitation count under the behavior policy: the entire state space is well-covered. }
\label{fig:visitation}
\end{figure}

Figure~\ref{fig:visitation} shows the visitation of inventory states under the numerical experiment settings documented in Section~\ref{section:numerical_experiment}. We observe that the uniformly randomized behavior policy explores the entire state space.

\begin{figure}[ht]

\centering
\includegraphics[width=\textwidth]{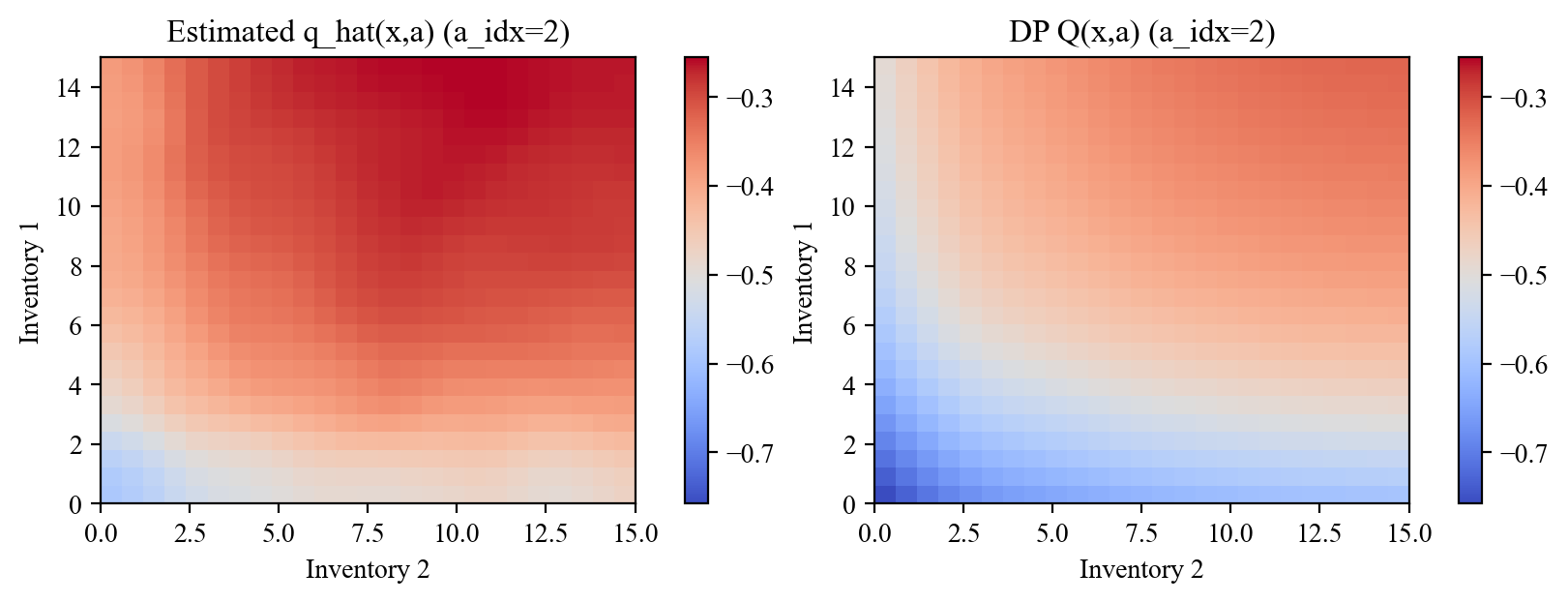}
\includegraphics[width=\textwidth]{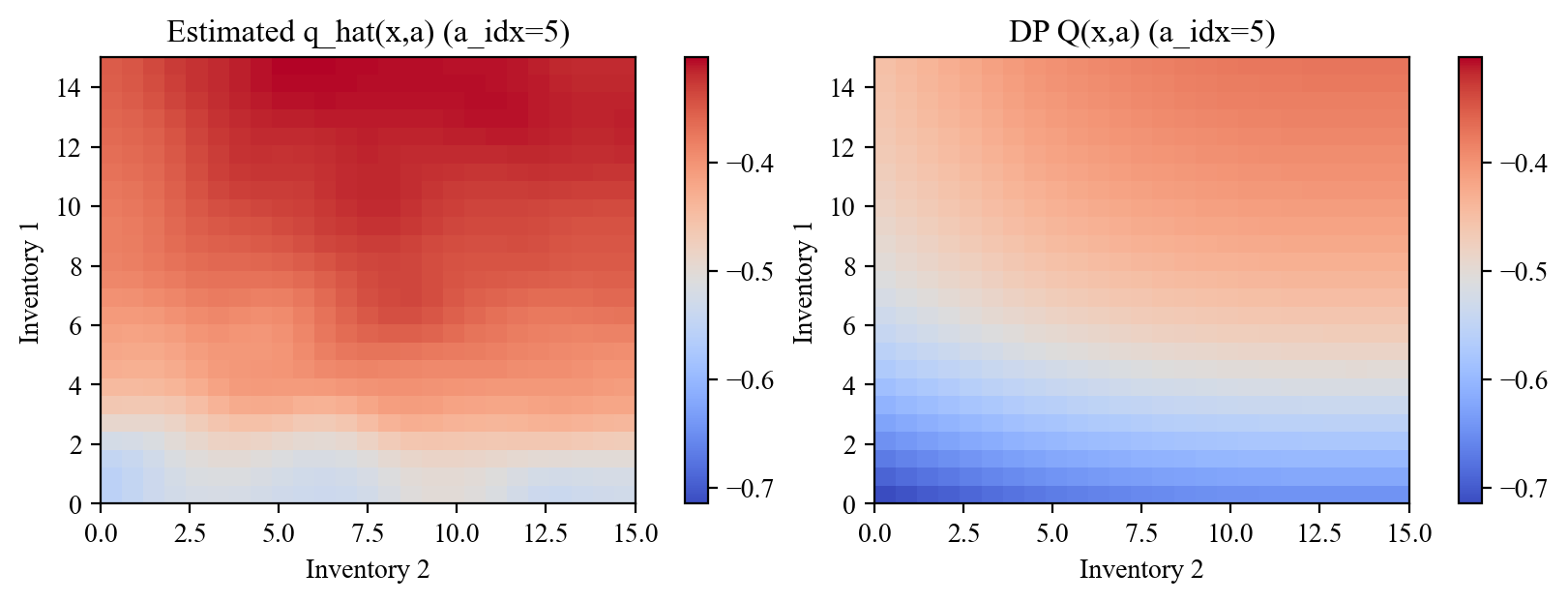}
\caption{Q-Measure estimate $q_n(\cdot,a)$ versus the DP benchmark, $a = (0,2)$ and $(0,5)$.}
\label{fig:inventory_q_slices}
\end{figure}

Figure~\ref{fig:inventory_q_slices} compares slices of the Q-functions $q_n(\cd,a)$ to the benchmark $Q_\mrm{DP}$ for two representative discrete actions $a=(0,2)$ and $a=(0,5)$. The learned surfaces match the benchmark qualitatively, but there is a visible global value shift of approximately $+0.1$, with $q_n$ overestimating $Q^*$. This indicates that the residual error of approximately $0.1$ in the RMSE curve in Figure~\ref{fig:inventory_value_rmse} is primarily driven by this global shift.

\FloatBarrier

\subsection{Learning Outcome Under Partial Coverage}
\begin{figure}[htbp]

\centering
\includegraphics[width=0.6\textwidth]{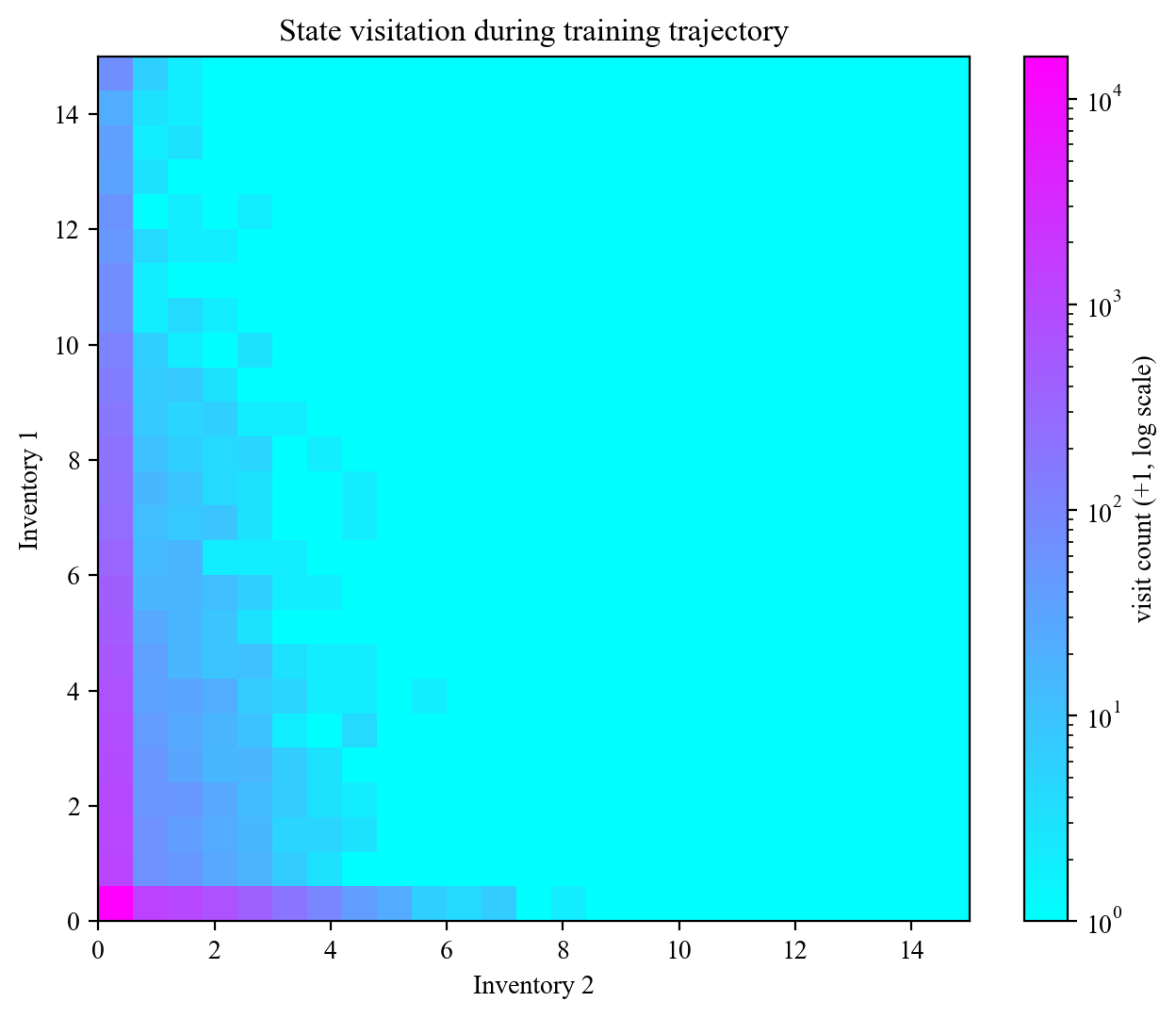}
\includegraphics[width=0.95\textwidth]{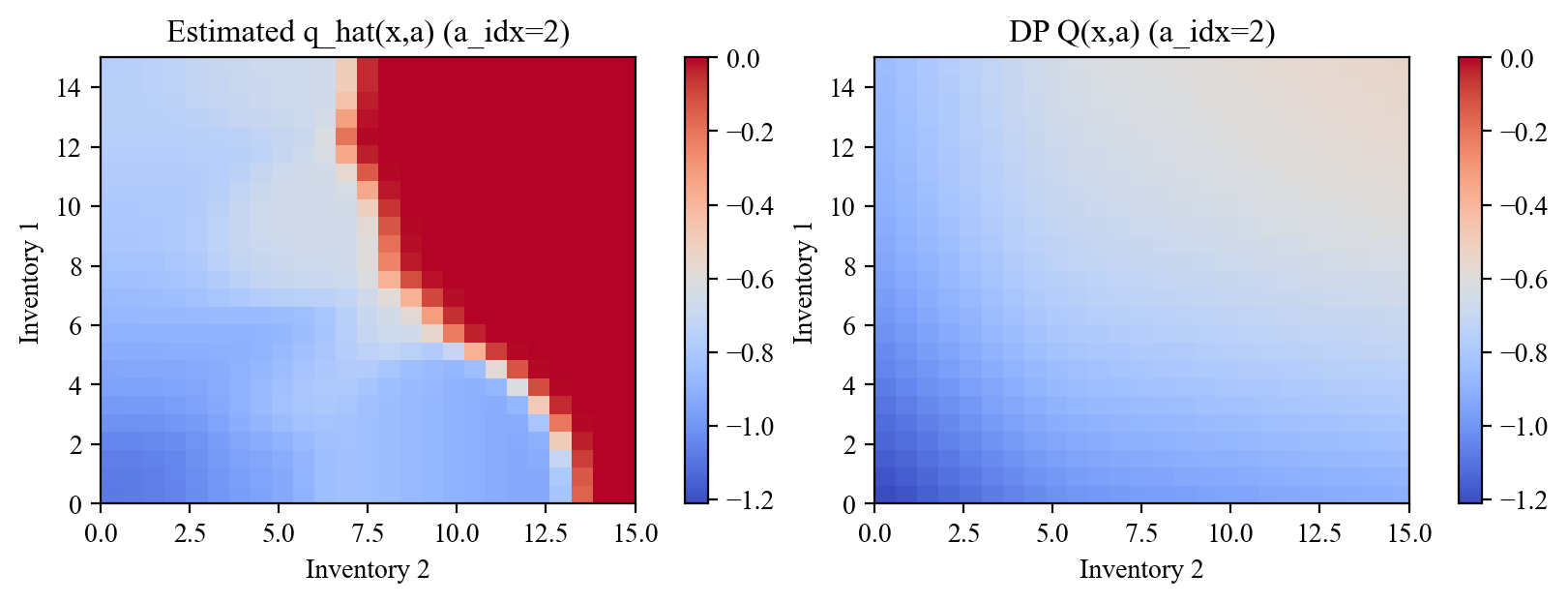}
\includegraphics[width=0.95\textwidth]{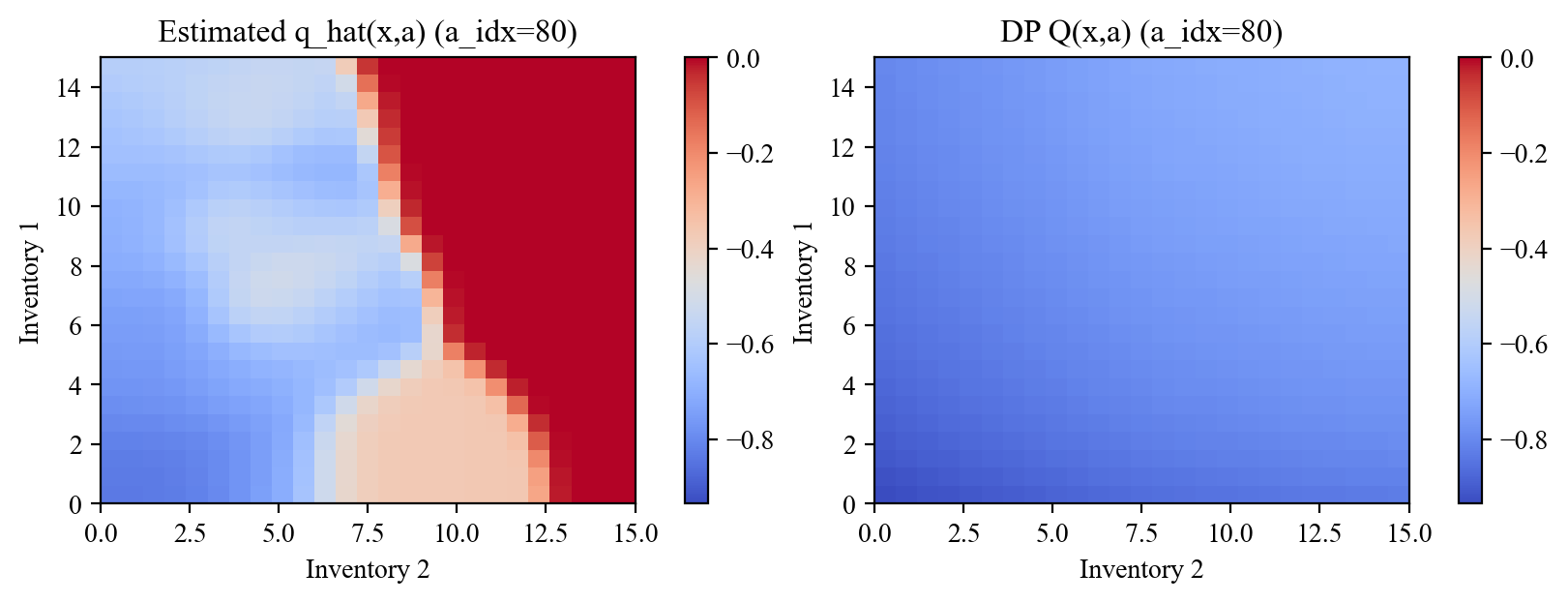}
\caption{Partial visitation of inventory states: the behavior policy explores only the bottom-left corner of the state space. The Q-measure estimates of $Q^*$ are accurate in the bottom-left region, but are off in the top-right region.}
\label{fig:partial_cov}
\end{figure}

On the other hand, we also implement our algorithm in a setting where the exploration policy is not able to explore the whole state space; see Figure~\ref{fig:partial_cov}. In this case, the demand $D = |G'|$ where $G'\sim N((8, 7),\Sigma)$ for the same $\Sigma$ as before. Thus, the demand is greater than in the earlier experimental setting.  We observe that the policy explores only the bottom-left corner of the state space. Correspondingly, the estimate $q_n$ of $Q^*$ remains accurate in regions that the behavior policy visits, but becomes unreliable in regions that the behavior policy is unable to reach.

\end{document}